\documentclass[10pt,twocolumn,letterpaper]{article}

\usepackage[pagenumbers]{cvpr} 

\usepackage{graphicx}
\usepackage{amsmath}
\usepackage{amssymb}
\usepackage{booktabs}
\usepackage{amsthm}
\usepackage[accsupp]{axessibility}  
\usepackage{longtable}

\usepackage{subcaption}

\makeatletter
\@namedef{ver@everyshi.sty}{}
\makeatother

\usepackage{color}
\usepackage{tikz}

\newtheorem{proposition}{Proposition}
\newtheorem{lemma}{Lemma}

\usepackage[pagebackref,breaklinks,colorlinks]{hyperref}

\usepackage[capitalize]{cleveref}
\crefname{section}{Sec.}{Secs.}
\Crefname{section}{Section}{Sections}
\Crefname{table}{Table}{Tables}
\crefname{table}{Tab.}{Tabs.}

\begin{document}

\title{Power Bundle Adjustment for Large-Scale 3D Reconstruction}


\author{Simon Weber\textsuperscript{1,2} \\ 
{\tt\small sim.weber@tum.de}
\and Nikolaus Demmel\textsuperscript{1,2} \\ 
{\tt\small nikolaus.demmel@tum.de}
\and Tin Chon Chan\textsuperscript{1,2} \\ 
{\tt\small tin-1254@hotmail.com}
\and Daniel Cremers\textsuperscript{1,2,3} \\ 
{\tt\small cremers@tum.de}
}

\maketitle
\footnotetext[1]{Technical University of Munich}
\footnotetext[2]{Munich Center for Machine Learning}
\footnotetext[3]{University of Oxford}
\begin{abstract}
We introduce Power Bundle Adjustment as an  expansion type algorithm for solving large-scale bundle adjustment problems. It is based on the power series expansion of the inverse Schur complement and constitutes a new family of solvers that we call inverse expansion methods. We theoretically justify the use of power series and we prove the convergence of our approach. Using the real-world BAL dataset we show that the proposed solver challenges the state-of-the-art iterative methods and significantly accelerates the solution of the normal equation, even for reaching a very high accuracy. This easy-to-implement solver can also complement a recently presented distributed bundle adjustment framework. We demonstrate that employing the proposed Power Bundle Adjustment as a sub-problem solver significantly improves speed and accuracy of the distributed optimization.
\end{abstract}

\section{Introduction}
\vspace*{-0.5mm}

Bundle adjustment (BA) is a classical computer vision problem that forms the core component of many 3D reconstruction and Structure from Motion (SfM) algorithms. It refers to the joint estimation of camera parameters and 3D landmark positions by minimization of a non-linear reprojection error. The recent emergence of large-scale internet photo collections \cite{key-1}  raises the need for BA methods that are scalable with respect to both runtime and memory. And building accurate city-scale maps for applications such as augmented reality or autonomous driving brings current BA approaches to their limits.

\begin{figure}[t!]
\begin{subfigure}{0.55\textwidth}
    \hspace{0.4cm}
    \includegraphics[width=0.75\linewidth, height=4.2cm]{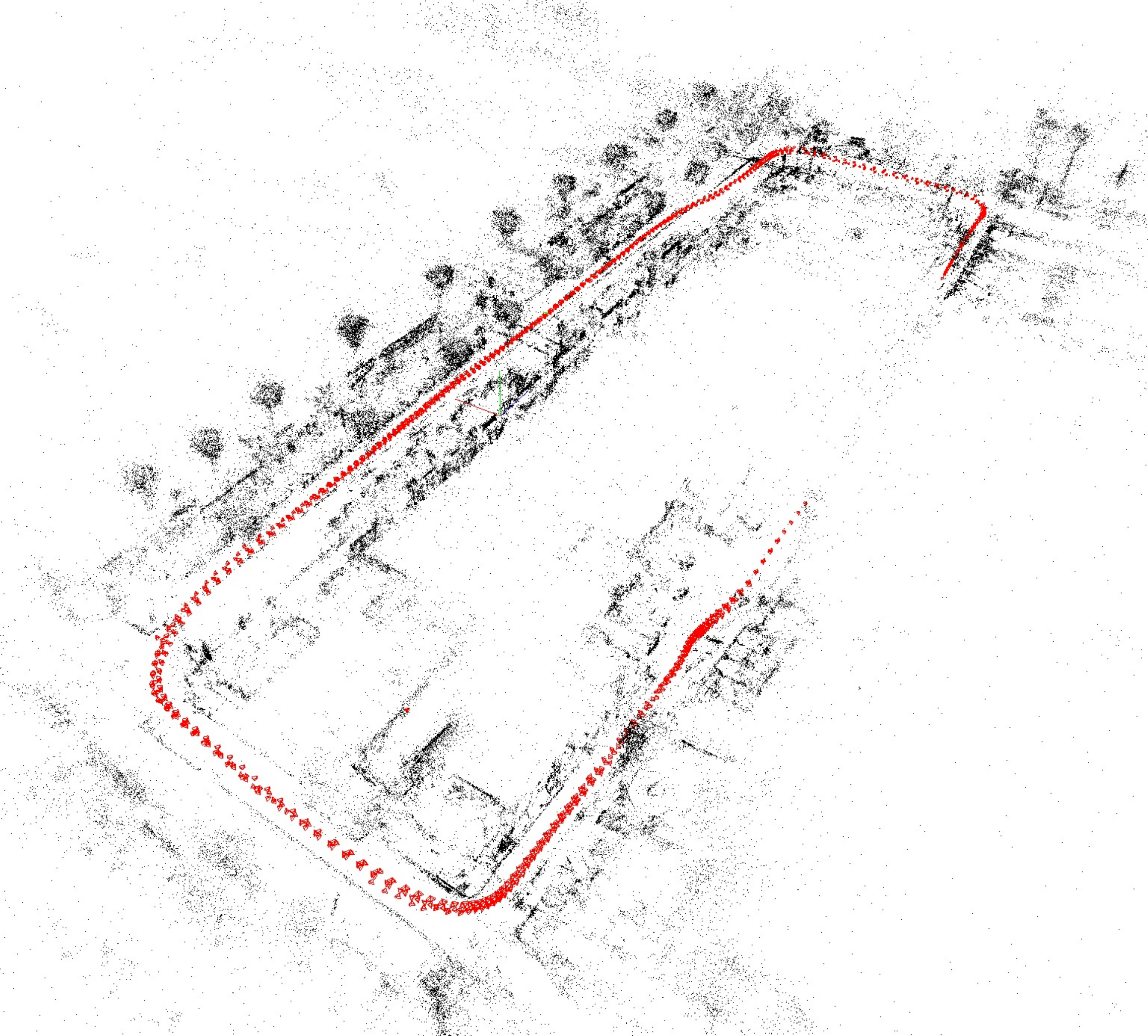} 
\caption{\textit{Ladybug-1197}}
\label{}
\end{subfigure}
\begin{subfigure}{0.45\textwidth}
    \hspace{0.4cm}
    \includegraphics[width=0.9\linewidth, height=4.5cm]{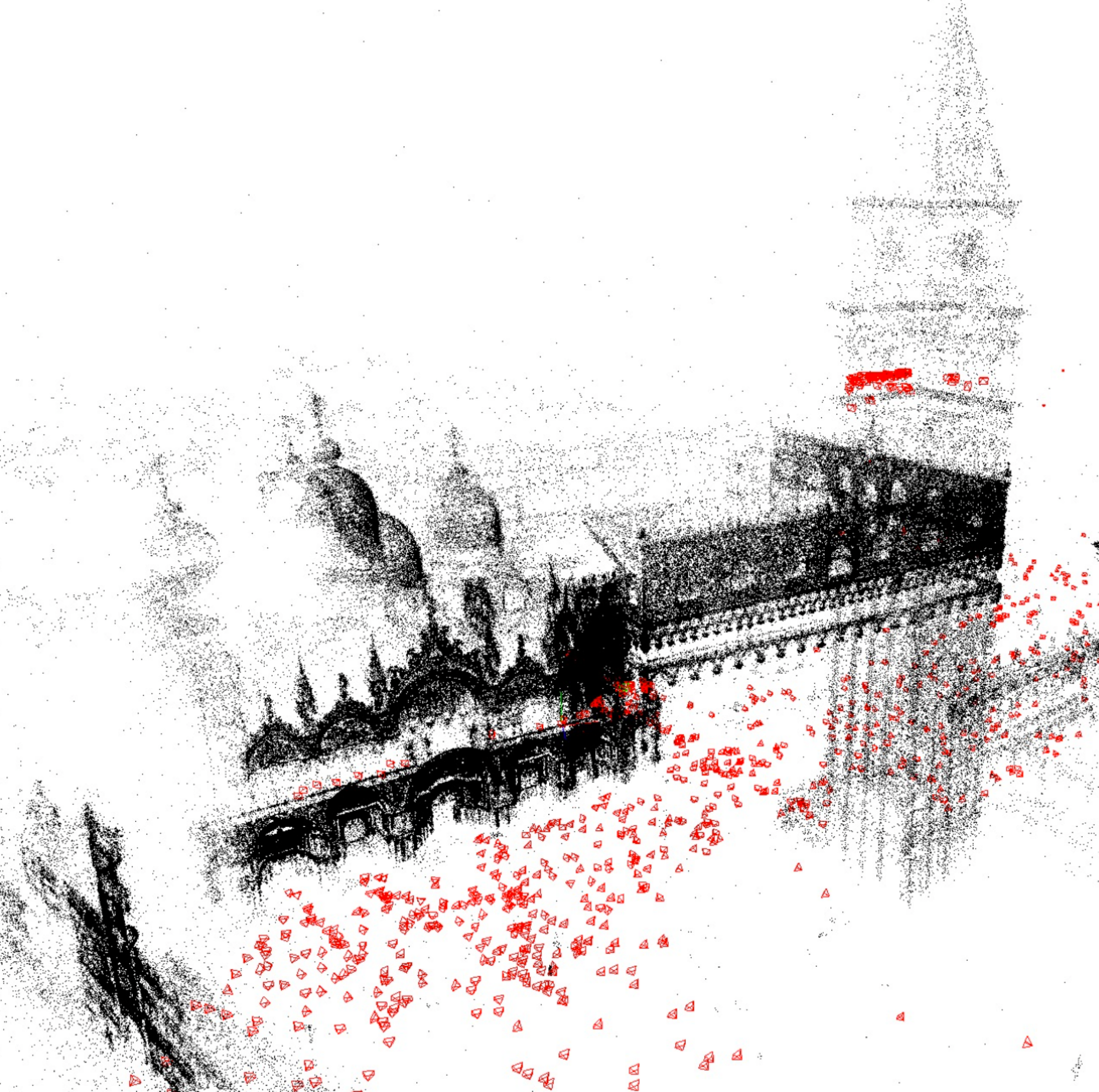}
\caption{\textit{Venice-1102}}
\label{}
\end{subfigure}
\caption{Power Bundle Adjustment (PoBA) is a novel solver for large-scale BA problems that is significantly faster and more memory-efficient than existing solvers. (a) Optimized 3D reconstruction of a \textit{Ladybug} BAL problem with $1197$ poses. \textit{PoBA-$32$} (resp. \textit{PoBA-$64$})  is $41\%$ (resp. $36\%$) faster than the best competing solver to reach a cost tolerance of $1\%$. (b) Optimized 3D reconstruction of a \textit{Venice} BAL problem with $1102$ poses. \textit{PoBA-$32$} (resp. \textit{PoBA-$64$}) is $71\%$ (resp. $69\%$) faster than the best competing solver to reach a cost tolerance of $1\%$. \textit{PoBA} is five times (resp. twice) less memory-consuming than $\sqrt{BA}$ (resp. Ceres).}
\label{fig:pangolin}
\end{figure}

As the solution of the normal equation is the most time consuming step of BA, the Schur complement trick is usually employed to form the reduced camera system (RCS). This linear system involves only the pose parameters and is significantly smaller. Its size can be reduced even more by using a QR factorization, deriving only a matrix square root of the RCS, and then solving an algebraically equivalent problem \cite{demmel2021rootba}. Both the RCS and its square root formulation are commonly solved by iterative methods such as the popular preconditioned conjugate gradients algorithm for large-scale problems or by direct methods such as Cholesky factorization for small-scale problems. 

In the following, we will challenge these two families of solvers by relying on an iterative approximation of the inverse Schur complement. In particular, our contributions are as follows:

\begin{itemize}
    \item[$\bullet$] We introduce Power Bundle Adjustment (\textit{PoBA}) for efficient large-scale BA. This new family of techniques that we call \textit{inverse expansion methods} challenges the state-of-the-art methods which are built on iterative and direct solvers.
    \item[$\bullet$] We link the bundle adjustment problem to the theory of power series and we provide theoretical proofs that justify this expansion and establish the convergence of our solver.
    \item[$\bullet$] We perform extensive evaluation of the proposed approach on the BAL dataset and compare to several state-of-the-art solvers. We highlight the benefits of \textit{PoBA} in terms of speed, accuracy, and memory-consumption. \Cref{fig:pangolin} shows reconstructions for two out of the 97 evaluated BAL problems.
    \item[$\bullet$] We incorporate our solver into a recently proposed distributed BA framework and show a significant improvement in terms of speed and accuracy.
    \item[$\bullet$] We release our solver as open source to facilitate further research: \url{https://github.com/simonwebertum/poba}
\end{itemize}

\section{Related Work}
Since we propose a new way to solve large-scale bundle adjustment problems, we will review works on bundle adjustment and on traditional solving methods, that is, direct and iterative methods. We also provide some background on power series. For a general introduction to series expansion we refer the reader to \cite{key-6}.

\subsection*{Scalable bundle adjustment.}
A detailed survey of bundle adjustment can be found in \cite{key-2}. The Schur complement \cite{key-12} is the prevalent way to exploit the sparsity of the BA Problem. The choice of resolution method is typically governed by the size of the normal equation: With increasing size, direct methods such as sparse and dense Cholesky factorization \cite{key-13} are outperformed by iterative methods such as inexact Newton algorithms. Large-scale bundle adjustment problems with tens of thousands of images are typically solved by the conjugate gradient method \cite{key-1,key-8,byrod2010conjugate}. Some variants have been designed, for instance the search-space can be enlarged \cite{key-3} or a visibility-based preconditioner can be used \cite{key-9}. A recent line of works on square root bundle adjustment proposes to replace the Schur complement for eliminating landmarks with nullspace projection \cite{demmel2021rootba,demmel2021rootvo}. It leads to significant performance improvements and to one of the most performant solver for the bundle adjustment problem in term of speed and accuracy. Nevertheless these methods still rely on traditional solvers for the reduced camera system, i.e.~preconditioned conjugate gradient method (PCG) for large-scale \cite{demmel2021rootba} and Cholesky decomposition for small-scale \cite{demmel2021rootvo} problems, besides an important cost in term of memory-consumption. Even with PCG, solving the normal equation remains the bottleneck and finding thousands of unknown parameters requires a large number of inner iterations. Other authors try to improve the runtime of BA with PCG by focusing on efficient parallelization \cite{ren2021megba}. Recently, Stochastic BA \cite{key-4} was introduced to stochastically decompose the reduced camera system into subproblems and solve the smaller normal equation by dense factorization. This leads to a distributed optimization framework with improved speed and scalability. By encapsulating the general power series theory into a linear solver we propose to simultaneously improve the speed, the accuracy and the memory-consumption of these existing methods. 


\subsection*{Power series solver.} While power series expansion is common to solve differential equations \cite{key-14}, to the best of our knowledge it has never been employed for solving the bundle adjustment problem. A recent work \cite{key-5} links the Schur complement to Neumann polynomial expansion to build a new preconditioner. Although this method presents interesting results for some physics problems such as convection-diffusion or atmospheric equations, it remains unsatisfactory for the bundle adjustment problem (see \Cref{total_time}). In contrast, we propose to directly apply the power series expansion of the inverse Schur complement for solving the BA problem. Our solver therefore falls in the category of expansion methods that -- to our knowledge -- have never been applied to the BA problem. In addition to being an easy-to-implement solver it leverages the special structure of the BA problem to simultaneously improve the trade-off speed-accuracy and the memory-consumption of the existing methods.

\begin{figure*}\label{fig:condition_number_poCG}
\begin{subfigure}[t]{0.33\textwidth}
    \includegraphics[scale=0.13]{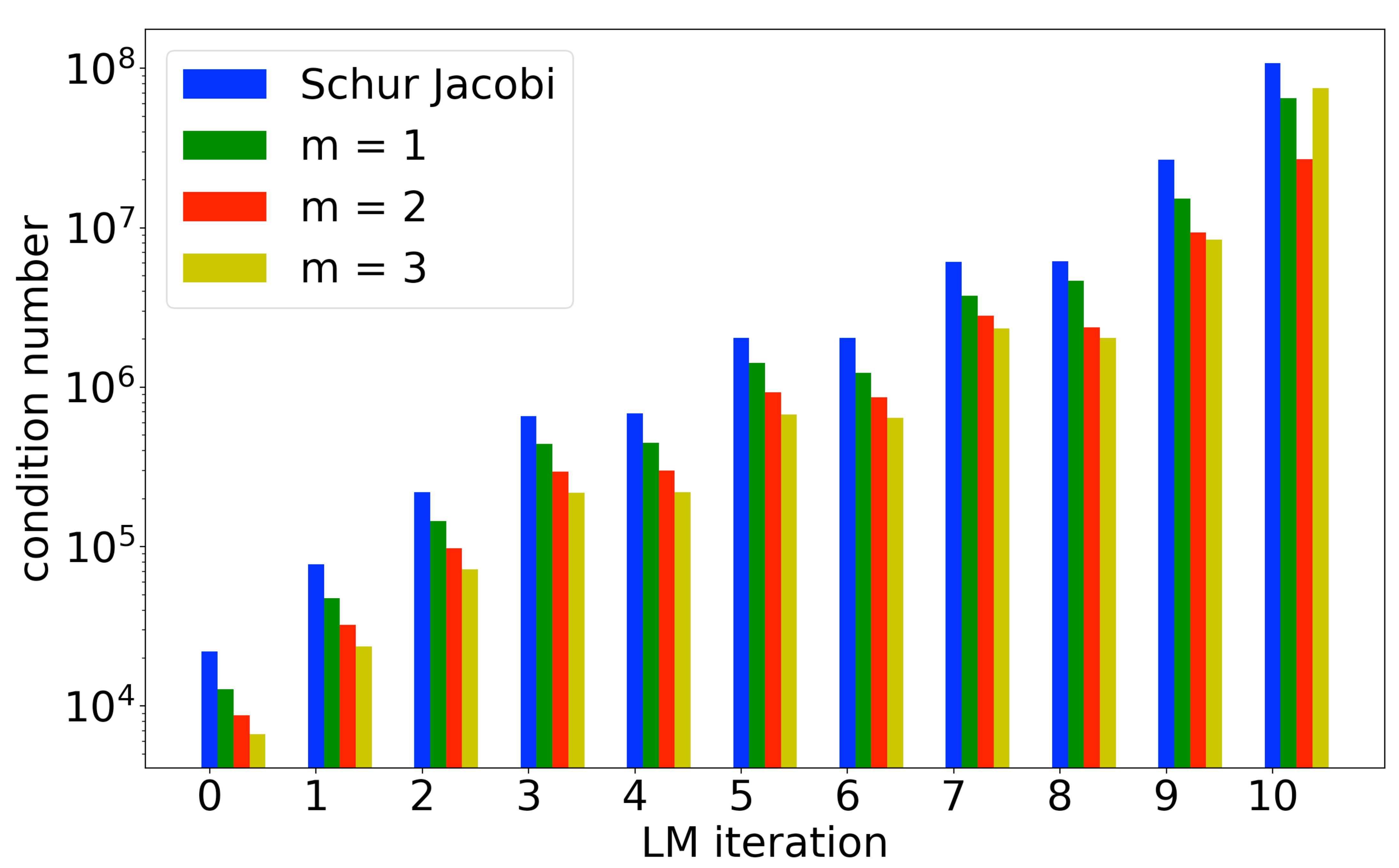} 
    \caption{Condition number}
    \label{fig:condition}
\end{subfigure}
\begin{subfigure}[t]{0.33\textwidth}
    \includegraphics[scale=0.13]{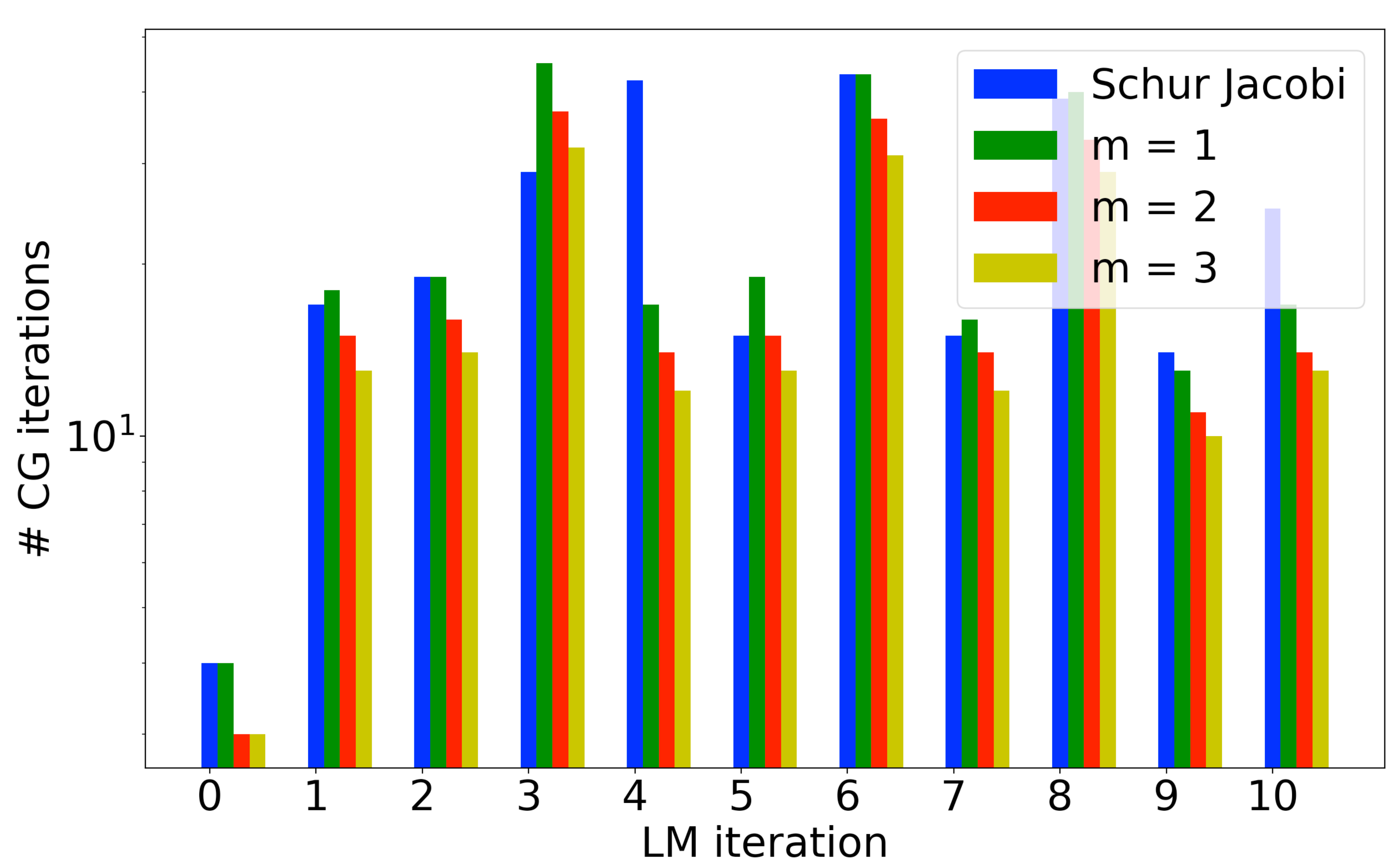} 
    \caption{Number of CG iterations}
    \label{fig:iteration}
\end{subfigure}
\begin{subfigure}[t]{0.33\textwidth}
    \includegraphics[scale=0.13]{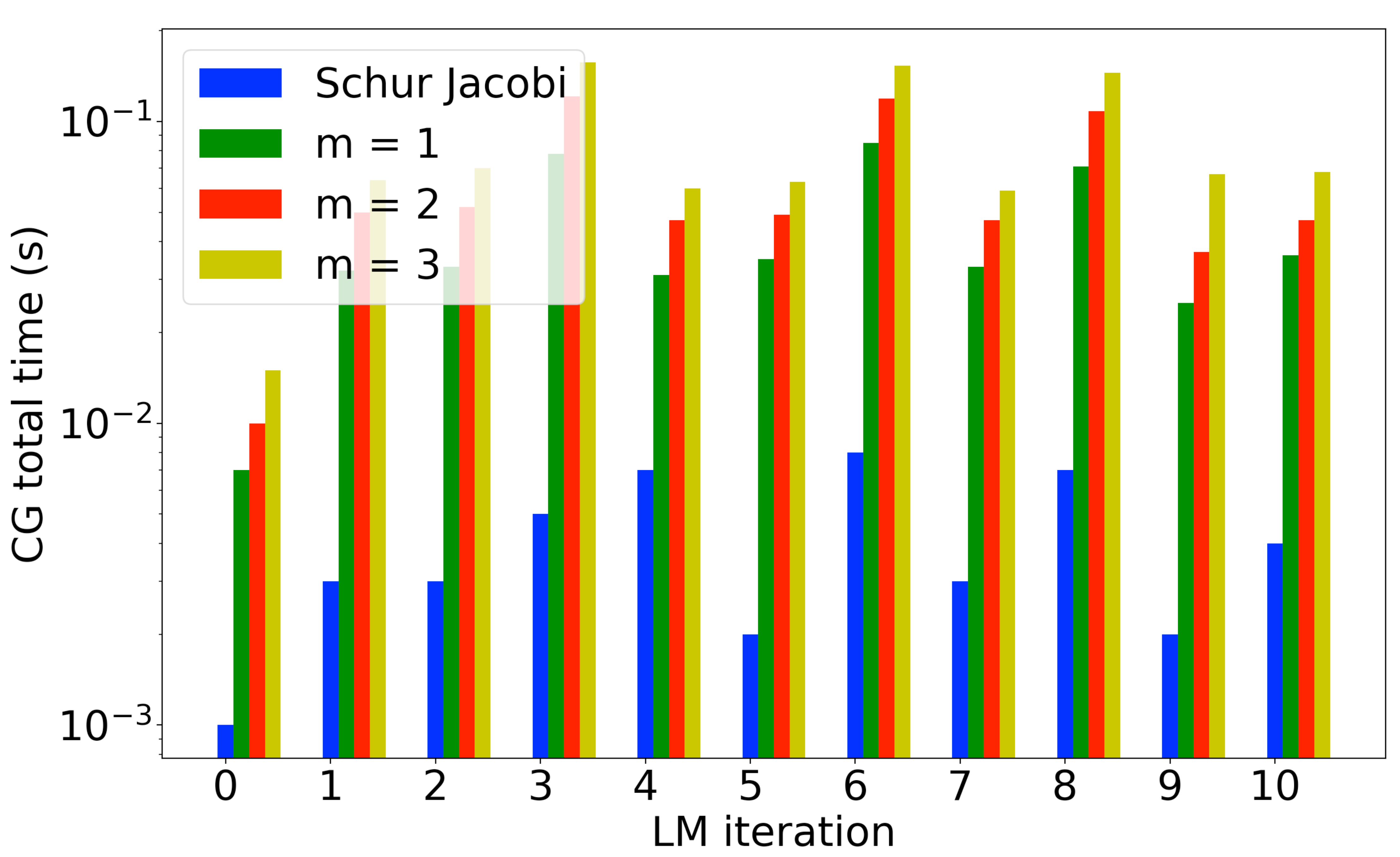}
    \caption{Total runtime of the CG algorithm}
    \label{fig:time}
\end{subfigure}
\caption{Although \cite{key-5} explores the use of power series as a preconditioner for some physics problems it suffers from the special structure of the BA formulation. Given a preconditioner $M^{-1}$ and the Schur complement $S$, the condition number $\kappa (M^{-1}S)$ is linked to the convergence of the conjugate gradients algorithm. (a) illustrates the behaviour of $\kappa$ for the ten first iterations of the LM algorithm for the real problem Ladybug-49 with $49$ poses from BAL dataset and for different orders $m$ of the power series expansion (\ref{S_approx}) used as preconditioner for the CG algorithm. The condition number associated to the popular Schur-Jacobi preconditioner is reduced with this power series preconditioner, that is illustrated by a better convergence of the CG algorithm and then a smaller number of CG iterations (b). Nevertheless each supplementary order $m$ is more costly in terms of runtime as the application of the power series preconditioner involves $4m$ matrix-vector product, whereas the Schur-Jacobi preconditioner can be efficiently stored and applied. (c) It leads to an increase of the overall runtime when solving the normal equation (\ref{normal}).} 
\label{total_time}
\end{figure*}

\section{Power Series}
We briefly introduce power series expansion of a matrix. Let $\rho(A)$ denote the spectral radius of a square matrix $A$, i.e. the largest absolute eigenvalue and denote the spectral norm by $\lVert A \rVert = \rho(A)$. The following proposition holds:
\begin{proposition}\label{proposition}
Let $M$ be a $n \times n$ matrix. If the spectral radius of $M$ satisfies $\lVert M \rVert <1$, then 
\begin{equation} 
(I-M)^{-1} = \sum_{i = 0}^{m}M^{i} + R \, ,
\end{equation} 
where the error matrix 
\begin{equation} 
R = \sum_{i = m+1}^{\infty}M^{i} \, ,
\end{equation} 
satisfies 
\begin{equation}\label{bounded}
\lVert R \rVert \leq \frac{\lVert M \rVert ^{m+1}}{1 - \lVert M \rVert}  \, .
\end{equation}
\end{proposition}
A proof is provided in Appendix and an illustration with real problems is given in \Cref{fig:boundary}. 
\section{Power Bundle Adjustment}
We consider a general form of bundle adjustment with $n_{p}$ poses and $n_{l}$ landmarks. Let $x = (x_{p}, x_{l})$ be the state vector containing all the optimization variables, where the vector $x_{p}$ of length $d_{p}n_{p}$ is associated to the extrinsic and (possibly) intrinsic camera parameters for all poses and the vector $x_{l}$ of length $3n_{l}$ is associated to the 3D coordinates of all landmarks. In case only the extrinsic parameters are unknown then $d_{p} = 6$ for rotation and translation of each camera. For the evaluated BAL problems we additionally estimate intrinsic parameters and $d_{p} = 9$. The objective is to minimize the total bundle adjustment energy 
\begin{equation} 
F(x) = \frac{1}{2} \lVert r(x) \rVert_2^{2} = \frac{1}{2} \sum_{i}\lVert r_{i}(x) \rVert_2^{2} \, , 
\end{equation} where the vector $r(x) = [r_{1}(x)^{\top},...,r_{k}(x)^{\top}]^{\top}$ comprises all residuals capturing the discrepancy between model and observation.
\subsection{Least Squares Problem}
This nonlinear least squares problem is commonly solved with the Levenberg-Marquardt (LM) algorithm, which is based on the first-order Taylor approximation of $r(x)$ around the current state estimate $x^{0}=(x_{p}^{0}, x_{l}^{0})$. By adding a regularization term to improve convergence the minimization turns into 
\begin{equation} 
\begin{split}
\min_{\Delta x_{p}, \Delta x_{l}} \frac{1}{2} \Big(\Big\lVert r^{0} + \begin{pmatrix} J_{p} & J_{l} \end{pmatrix} \begin{pmatrix} \Delta x_{p} \\ \Delta x_{l} \end{pmatrix} \Big\rVert_{2} ^{2} \\ + \lambda \Big\lVert \begin{pmatrix} D_{p} & D_{l} \end{pmatrix} \begin{pmatrix} \Delta x_{p} \\ \Delta x_{l} \end{pmatrix} \Big\rVert_2^{2}\Big) \, ,
\end{split}
\end{equation}
with $r^{0} = r(x^{0})$, $J_{p}=\frac{\partial r}{\partial x_{p}}|_{x^{0}}$, $J_{l}=\frac{\partial r}{\partial x_{l}}|_{x^{0}} $, $\lambda$ a damping coefficient, and $D_{p}$ and $D_{c}$ diagonal damping matrices for pose and landmark variables. This damped problem leads to the corresponding normal equation
\begin{equation}\label{normal}
H \begin{pmatrix} \Delta x_{p} \\ \Delta x_{l} \end{pmatrix} = - \begin{pmatrix} b_{p} \\ b_{l} \end{pmatrix} \, ,
\end{equation} 
where 

\begin{align}
H = \begin{pmatrix} 
U_{\lambda} & W \\ W^{\top} & V_{\lambda}\end{pmatrix}, \\
U_{\lambda} = J_{p}^{\top}J_{p} + \lambda D_{p}^{\top}D_{p}, \\ V_{\lambda} = J_{l}^{\top}J_{l} + \lambda D_{l}^{\top}D_{l}, \\  W = J_{p}^{\top}J_{l}, \\ b_{p} = J_{p}^{\top}r^{0},\; b_{l} = J_{l}^{\top}r^{0} \, .
\end{align}
$U_{\lambda}$, $V_{\lambda}$ and $H$ are symmetric positive-definite \cite{key-2}. 

\subsection{Schur Complement}
As inverting the system matrix $H$ of size $(d_{p}n_{p} + 3n_{l})^{2}$ directly tends to be excessively costly for large-scale problems it is common to reduce it by using the Schur complement trick. The idea is to form the reduced camera system 
\begin{equation}\label{Sx}
S \Delta x_{p} = - \tilde{b} \, ,
\end{equation} 
with 
\begin{align} 
S = U_{\lambda}-WV_{\lambda}^{-1}W^{\top}, \label{Schur} \\
\tilde{b} = b_{p} - WV_{\lambda}^{-1}b_{l} \, .
\end{align}
(\ref{Sx}) is then solved for $\Delta x_{p}$. The optimal $\Delta x_{l}$ is obtained by back-substitution:
\begin{equation}\label{back_substitution}
\Delta x_{l} = -V_{\lambda}^{-1}(-b_{l}+W^{\top}\Delta x_{p}) \, .
\end{equation}

\subsection{Power Bundle Adjustment}
Factorizing (\ref{Schur}) with the block-matrix $U_{\lambda}$ 
\begin{equation}
    S = U_{\lambda}(I - U_{\lambda}^{-1}WV_{\lambda}^{-1}W^{\top})
\end{equation}
leads to formulate the inverse Schur complement as
\begin{equation}\label{inverse} 
S^{-1}=(I - U_{\lambda}^{-1}WV_{\lambda}^{-1}W^{\top})^{-1}U_{\lambda}^{-1} \, .
\end{equation}
In order to expand (\ref{inverse}) into a power series as detailed in Proposition \ref{proposition}, we require to bound the spectral radius of $U_{\lambda}^{-1}WV_{\lambda}^{-1}W^{\top}$ by $1$. 

By leveraging the special structure of the BA problem we prove an even stronger result:
\begin{lemma} 
Let $\mu$ be an eigenvalue of $U_{\lambda}^{-1}WV_{\lambda}^{-1}W^{\top}$. Then 
\begin{equation}
    \mu \in [0,1[ \, .
\end{equation}
\end{lemma}
\begin{proof}
On the one hand $U_{\lambda}^{-\frac{1}{2}}WV_{\lambda}^{-1}W^{\top}U_{\lambda}^{-\frac{1}{2}}$ is symmetric positive semi-definite, as $U_{\lambda}$ and $V_{\lambda}$ are symmetric positive definite. Then its eigenvalues are greater than $0$. As $U_{\lambda}^{-\frac{1}{2}}WV_{\lambda}^{-1}W^{\top}U_{\lambda}^{-\frac{1}{2}}$ and $U_{\lambda}^{-1}WV_{\lambda}^{-1}W^{\top}$ are similar, 
\begin{equation}
    \mu \geq 0 \, .
\end{equation} On the other hand $U_{\lambda}^{-\frac{1}{2}}SU_{\lambda}^{-\frac{1}{2}}$ is symmetric positive definite as $S$ and $U_{\lambda}$ are. It follows that the eigenvalues of $U_{\lambda}^{-1}S$ are all strictly positive due to its similarity with $U_{\lambda}^{-\frac{1}{2}}SU_{\lambda}^{-\frac{1}{2}}$. As 
\begin{equation}
    U_{\lambda}^{-1}WV_{\lambda}^{-1}W^{\top} = I - U_{\lambda}^{-1}S \, ,
\end{equation}
it follows that 
\begin{equation}
    \mu < 1 \, ,
\end{equation} 
that concludes the proof. 
\end{proof}

Let be 
\begin{equation}\label{S_approx}
    \tilde{S}_{-1}(m) = \sum_{i=0}^{m}(U_{\lambda}^{-1}WV_{\lambda}^{-1}W^{\top})^{i}U_{\lambda}^{-1} \, ,
\end{equation}
and 
\begin{equation}\label{x_approx}
    x(m) = - \tilde{S}_{-1}(m) \tilde{b} \, ,
\end{equation}
for $m \geq 0$. The following proposition confirms that the approximation indeed converges with increasing order of $m$:

\begin{proposition}\label{prop_convergence}
$\lVert x(m) - \Delta x_{p} \rVert _{2} \underset{m \to +\infty}{\longrightarrow} 0 \, .$ 
\end{proposition}
\begin{proof}
We denote $P = U_{\lambda}^{-1}WV_{\lambda}^{-1}W^{\top}$. Due to Lemma 1
\begin{equation}\label{P_bound}
    \lVert P \rVert < 1 \, .
\end{equation}
The inverse Schur complement associated to (\ref{normal}) admits a power series expansion: 
\begin{equation}
    S^{-1} = \tilde{S}_{-1}(m) + R_{m} \, ,
\end{equation}
where
\begin{equation}
    R_{m} = \sum_{i = m+1}^{\infty}{P^{i}}U_{\lambda}^{-1}  
\end{equation}
satisfies 
\begin{equation}\label{R_S}
    \lVert R_{m} \lVert \leq \frac{\lVert P \rVert ^{m+1}}{1 - \lVert P \rVert} \lVert U_{\lambda}^{-1} \rVert \, .
\end{equation}
It follows that:
\begin{equation}
      x(m) - \Delta x_{p} = R_{m} \tilde{b} \, .
\end{equation}
The consistency of the spectral norm with respect to the vector norm implies: 
\begin{equation}\label{R_inequality}
    \lVert R_{m} \tilde{b} \rVert _{2} \leq \lVert R_{m} \rVert \lVert \tilde{b} \rVert _{2} \, .
\end{equation}
From (\ref{P_bound}), (\ref{R_S}) and (\ref{R_inequality}) we conclude the proof: 
\begin{equation}
    \lVert R_{m} \tilde{b} \rVert _{2} \underset{m \to +\infty}{\longrightarrow} 0 \, ,
\end{equation} 
and then
\begin{equation}
    \lVert x(m) - \Delta x_{p} \rVert _{2} \underset{m \to +\infty}{\longrightarrow} 0 \, .
\end{equation}
\end{proof}

This convergence result proves that 
\begin{itemize}
    \item an approximation of $\Delta x_{p}$ can be directly obtained by applying (\ref{S_approx}) to the right-hand side of (\ref{Sx});
    \item the quality of this approximation depends on the order $m$ and can be as small as desired.
\end{itemize}

\begin{figure*}
\begin{center}
\includegraphics[scale=0.61]{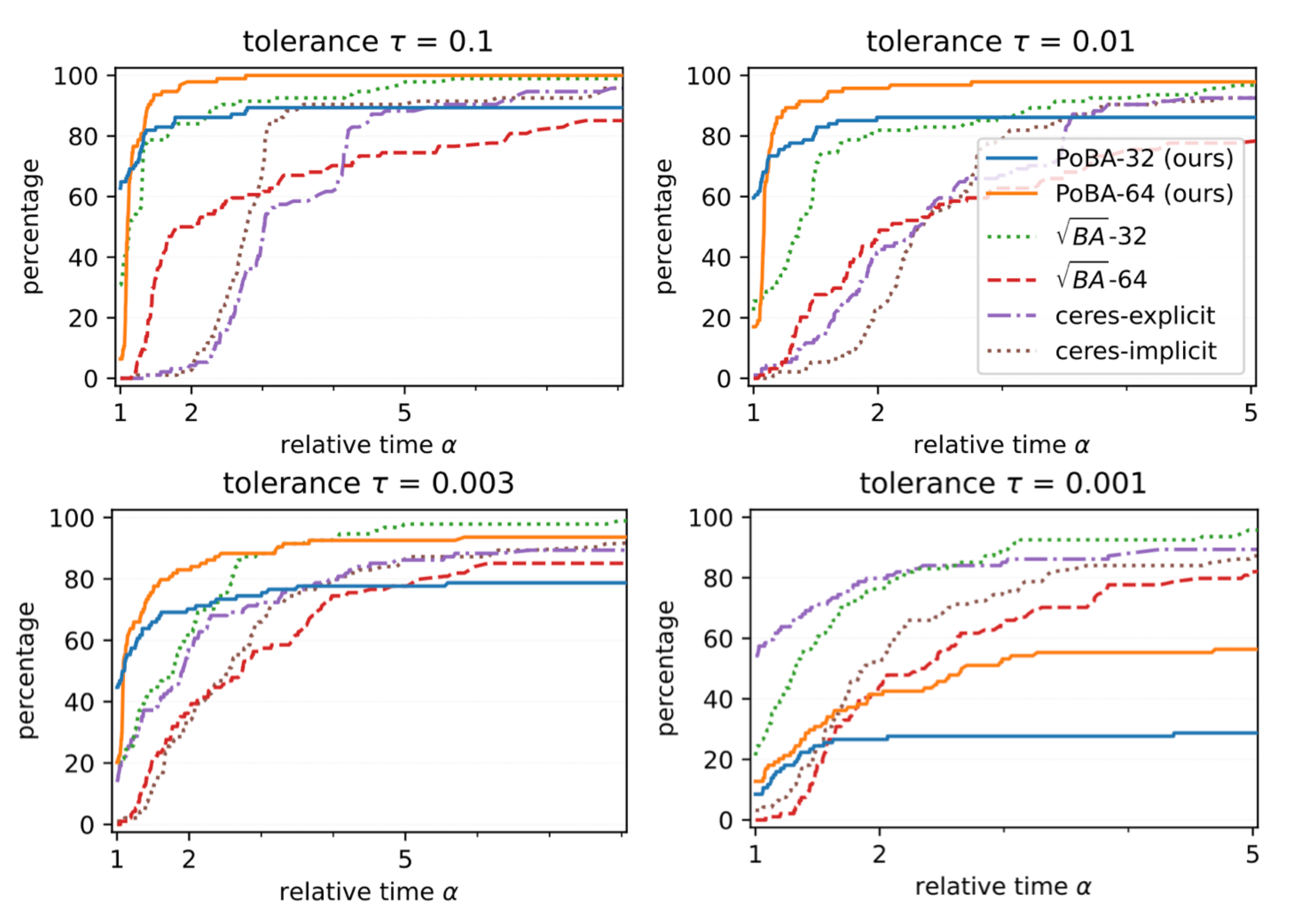}
\caption{Performance profiles for all BAL problems show the percentage of problems solved to a given accuracy tolerance $\tau \in \{0.1, 0.01, 0.003, 0.001\}$ with relative runtime $\alpha$. Our proposed solver \textit{PoBA} using series expansion of the Schur complement significantly outperforms all the competing solvers up to the high accuracy $\tau = 0.003$.}
\label{fig:performance}
\end{center}
\end{figure*}

The power series expansion being iteratively derived, a termination rule is necessary.

By analogy with inexact Newton methods \cite{key-7, nash, nash2} such that the conjugate gradients algorithm we set a stop criterion 
\begin{equation}\label{qcriterion}
    (i + 1) * \lVert (x(i) - x(i-1)) \rVert_{2} / \lVert x(i) \rVert_{2} < \epsilon \, ,
\end{equation}
for a given $\epsilon$. This criterion ensures that the power series expansion stops when the refinement of the pose update by expanding the inverse Schur complement into a supplementary order 
\begin{equation}
    \lVert (x(i) - x(i-1)) \rVert_{2}
\end{equation}
is much smaller than the average refinement when reaching the same order
\begin{equation}
    \frac{\lVert \sum_{j=1}^{i}{(x(j) - x(j-1)) +x(0)} \rVert _{2} }{i+1} = \frac{\lVert x(i) \rVert_{2}}{i+1} \, .
\end{equation}

\begin{figure}[tb]\label{consumption}
\begin{center}
\includegraphics[scale=0.3]{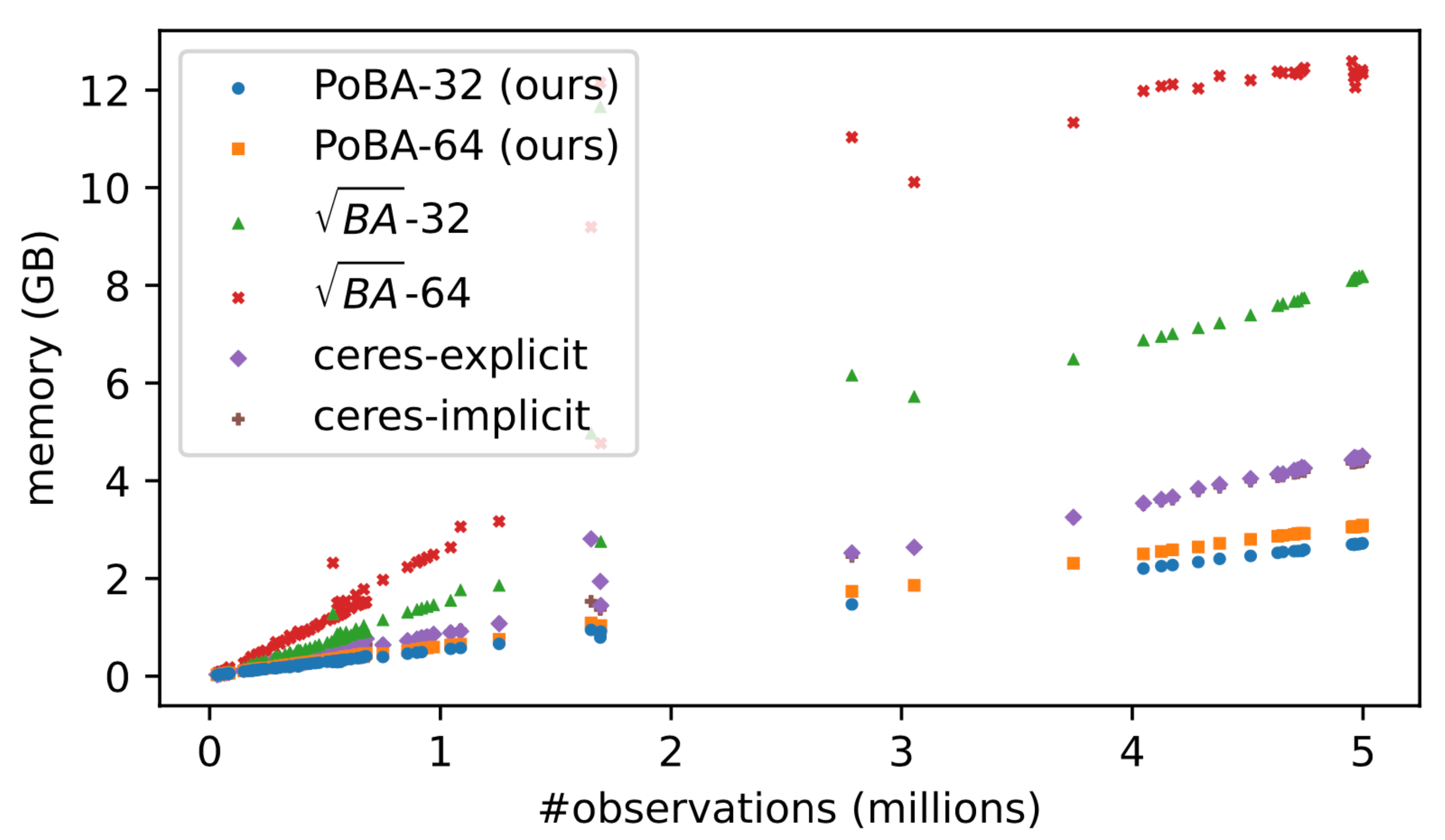}
\caption{Memory consumption for all BAL problems. The proposed \textit{PoBA} solver (orange and blue points) is five times less memory-consuming than $\sqrt{BA}$ solvers.}
\label{fig:consumption}
\end{center}
\end{figure}

\section{Implementation}\label{implementation}
We implement our \textit{PoBA} solver in C++ in single (\textit{PoBA-$32$}) and double (\textit{PoBA-$64$}) floating-point precision, directly on the publicly available implementation\footnote{\url{https://github.com/NikolausDemmel/rootba}} of \cite{demmel2021rootba}. This recent solver presents excellent performance to solve the bundle adjustment by using a QR factorization of the landmark Jacobians. It notably competes the popular Ceres solver. We additionally add a comparison with Ceres' sparse Schur complement solvers, similarly as in \cite{demmel2021rootba}. \textit{Ceres-explicit} and \textit{Ceres-implicit} iteratively solve (\ref{Sx}) with the conjugate gradients algorithm preconditioned by the Schur-Jacobi preconditioner. The first one saves $S$ in memory as a block-sparse matrix, the second one computes $S$ on-the-fly during iterations. $\sqrt{BA}$ and Ceres offer very competitive performance to solve the bundle adjustment problem, that makes them very challenging baselines to compare \textit{PoBA} to. We run experiments on MacOS 11.2 with an Intel Core i5 and 4 cores at 2GHz.

\begin{figure}[tb]

\begin{subfigure}{0.50\textwidth}
\includegraphics[width=1\linewidth, height=5cm]{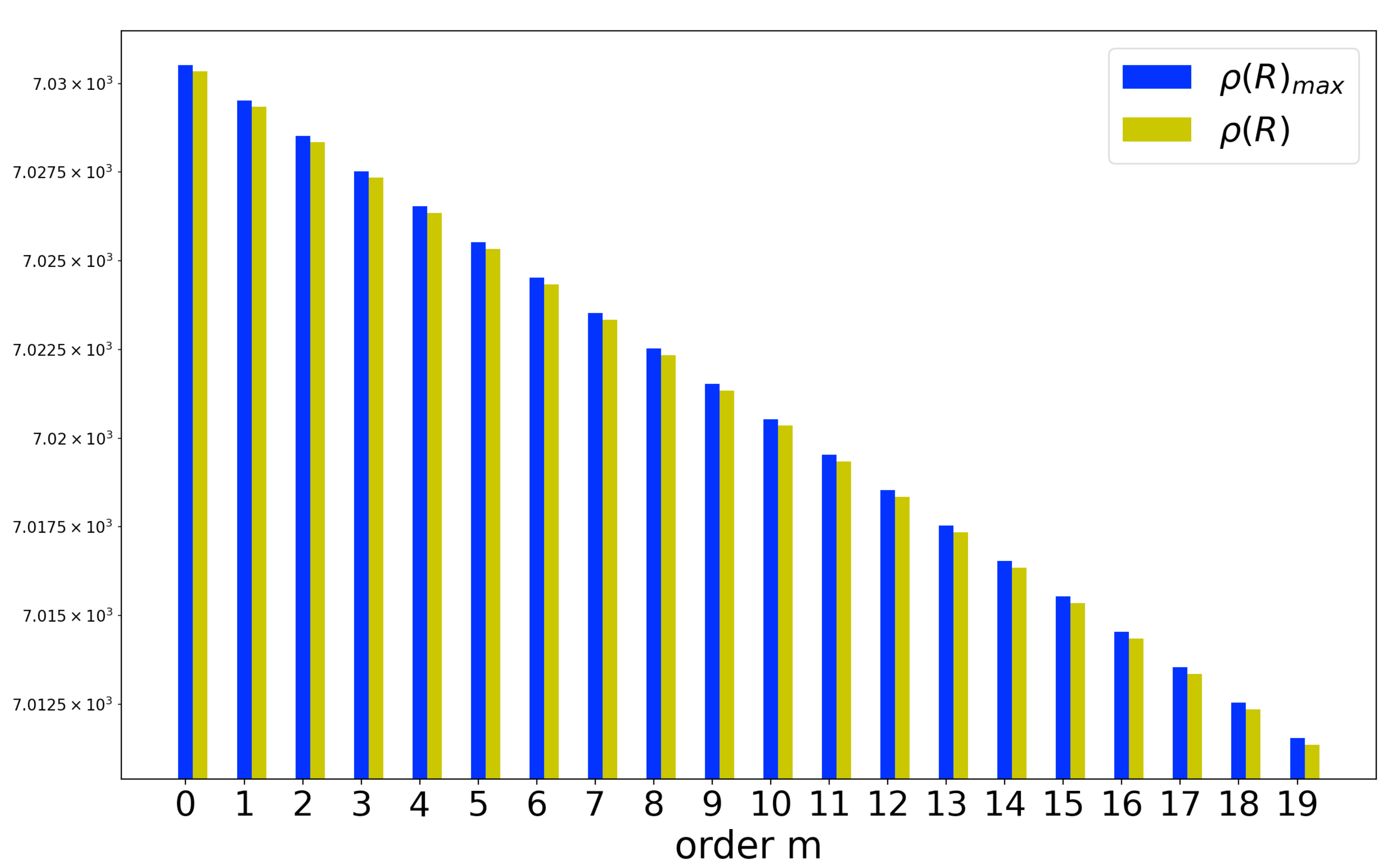} 
\caption{\textit{Ladybug-49}}
\label{fig:subim1}
\end{subfigure}
\begin{subfigure}{0.50\textwidth}
\includegraphics[width=1\linewidth, height=5cm]{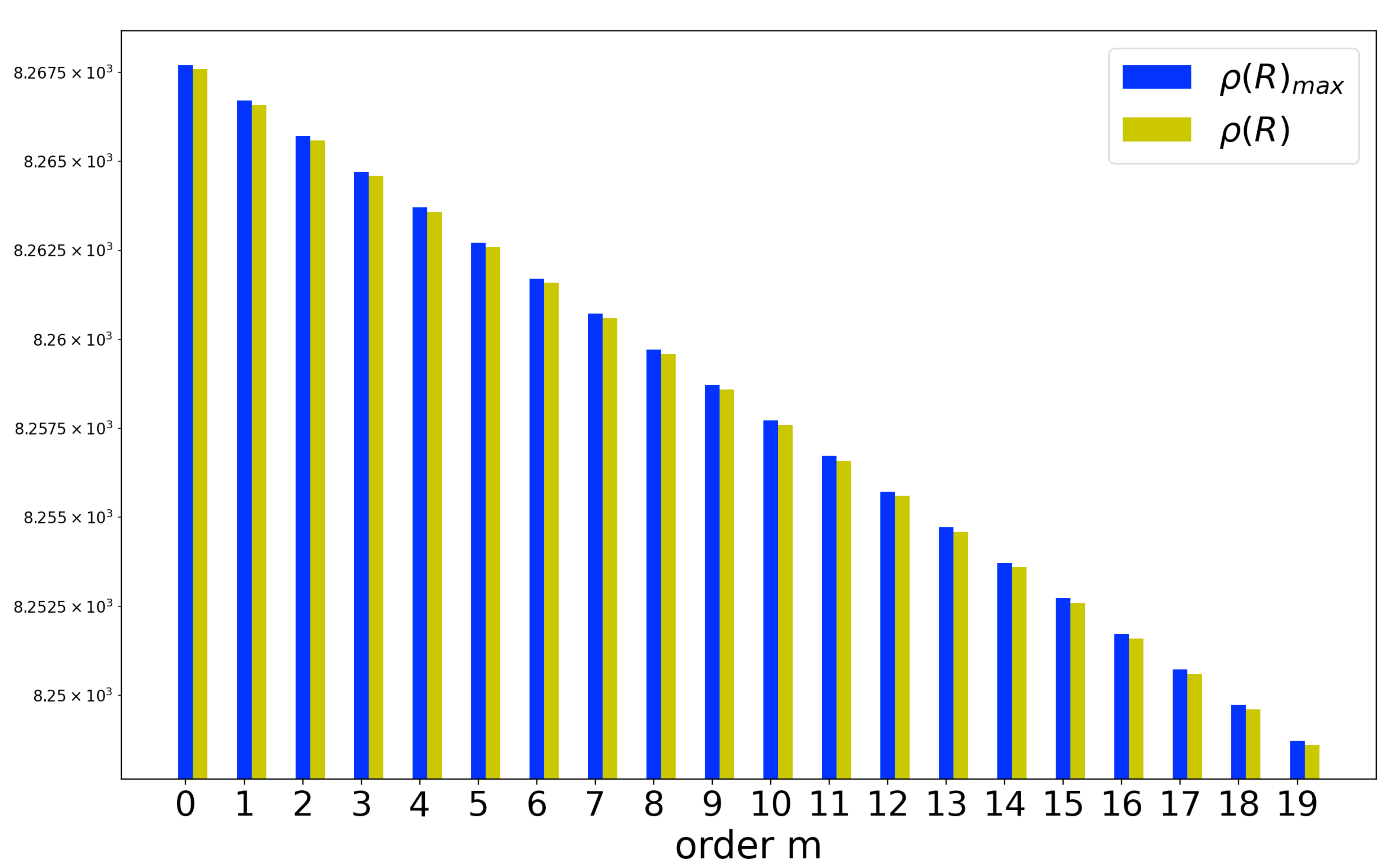}
\caption{\textit{Trafalgar-193}}
\label{fig:subim2}
\end{subfigure}
\caption{Illustration of the inequality (\ref{bounded}) in Proposition 1 for the first LM iteration of two BAL problems: (a) \textit{Ladybug} with $49$ poses and (b) \textit{Trafalgar} with $193$ poses. The spectral norm of the error matrix $R$ is plotted in green for $m < 20$. The right-side of the inequality plotted in blue represents the theoretical upper bound of the spectral norm of the error matrix and depends on the considered $m$ and on the spectral norm of $M = U_{\lambda}^{-1}WV_{\lambda}^{-1}W^{\top}$. With Spectra library \cite{key-18} $\rho(M)$ takes the values (a) $0.999858$ for \textit{L-49} and (b) $0.999879$ for \textit{T-193}. Both values are smaller than $1$ and $\rho (R)$ is always smaller than $\rho (M)^{m+1} / (1 - \rho (M))$, as stated in Lemma 1.} 
\label{fig:boundary}
\end{figure}

\subsubsection*{Efficient storage.}
We leverage the special structure of BA problem and design a memory-efficient storage. We group the Jacobian matrices and residuals by landmarks and store them in separate dense memory blocks. For a landmark with $k$ observations, all pose Jacobian blocks of size $2 \times d_p$ that correspond to the poses where the landmark was observed, are stacked and stored in a memory block of size $2k \times d_p$. Together with the landmark Jacobian block of size $2k \times 3$ and the residuals of length $2k$ that are also associated to the landmark, all information of a single landmark is efficiently stored in a memory block of size $2k \times (d_p+4)$. Furthermore, operations involved in (\ref{back_substitution}) and (\ref{x_approx}) are parallelized using the memory blocks.

\subsubsection*{Performance Profiles.} To compare a set of solvers the user may be interested in two factors, a lower runtime and a better accuracy. Performance profiles \cite{key-10} evaluate both jointly. Let $S$ and $P$ be respectively a set of solvers and a set of problems. Let $f_{0}(p)$ be the initial objective and $f(p,s)$ the final objective that is reached by solver $s \in S$ when solving problem $p \in P$. The minimum objective the solvers in $S$ attain for a problem $p$ is $f^{*}(p) = \min_{s \in S}f(p,s)$. Given a tolerance $\tau \in (0,1)$ the objective threshold for a problem $p$ is given by 
\begin{equation}
    f_{\tau}(p) = f^{*}(p) + \tau(f^{0}(p) - f^{*}(p))
\end{equation}
and the runtime a solver $s$ needs to reach this threshold is noted $T_{\tau}(p,s)$. It is clear that the most efficient solver $s^{*}$ for a given problem $p$ reaches the threshold with a runtime $T_{\tau}(p,s^{*}) = \min_{s \in S}T_{\tau}(p,s)$. Then, the performance profile of a solver for a relative runtime $\alpha$ is defined as 
\begin{equation}
\rho(s,\alpha) = \frac{100}{|P|} |\{p \in P | T_{\tau}(p,s) \leq \alpha \min_{s \in S} T_{\tau}(p,s)\}|
\end{equation}
Graphically the performance profile of a given solver is the percentage of problems solved faster than the relative runtime $\alpha$ on the x-axis.

\begin{figure*}
\begin{subfigure}{0.50\textwidth}
\includegraphics[width=1\linewidth, height=6.3cm]{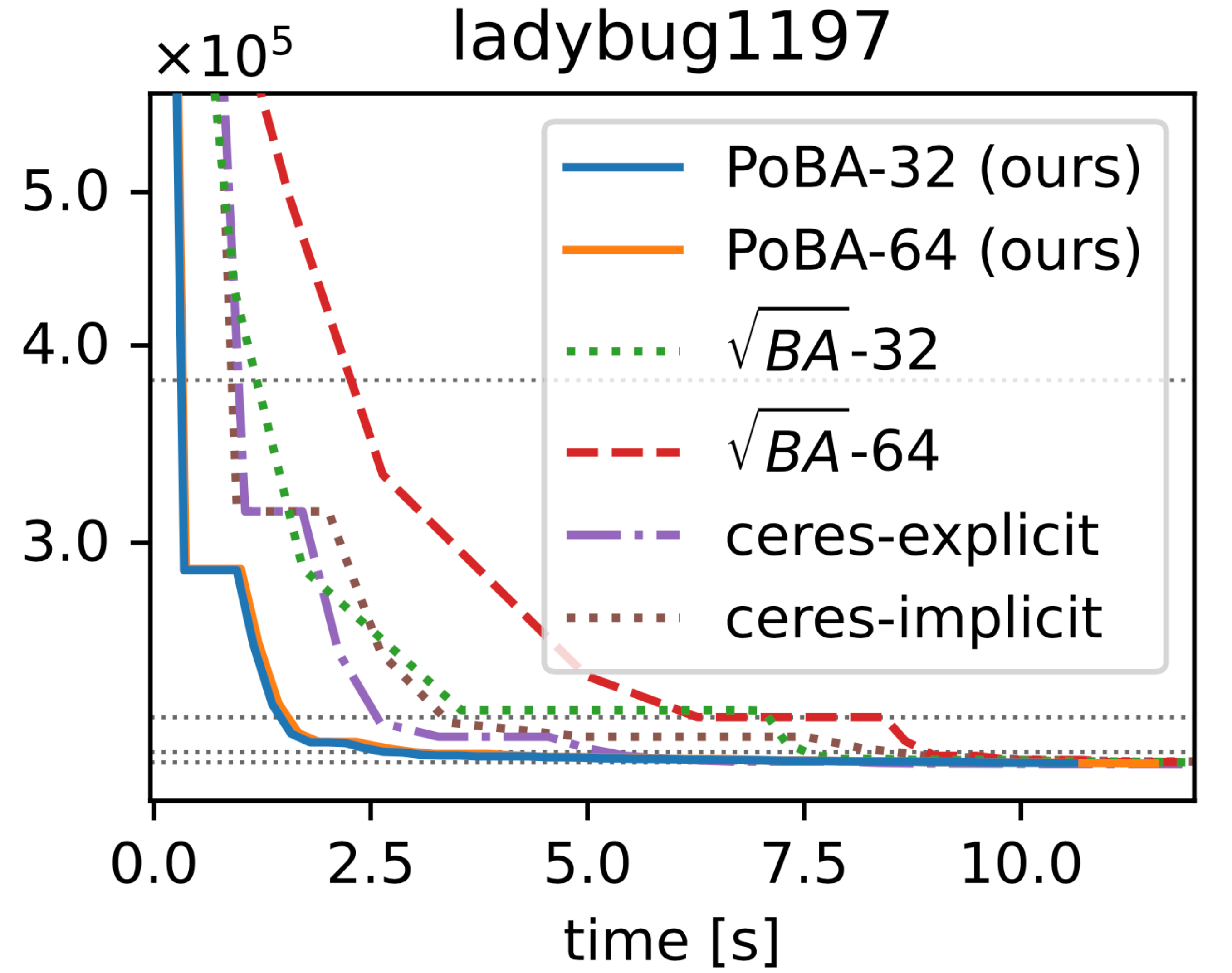} 
\caption*{}
\label{fig:subim1}
\end{subfigure}
\begin{subfigure}{0.50\textwidth}
\includegraphics[width=1\linewidth, height=6.3cm]{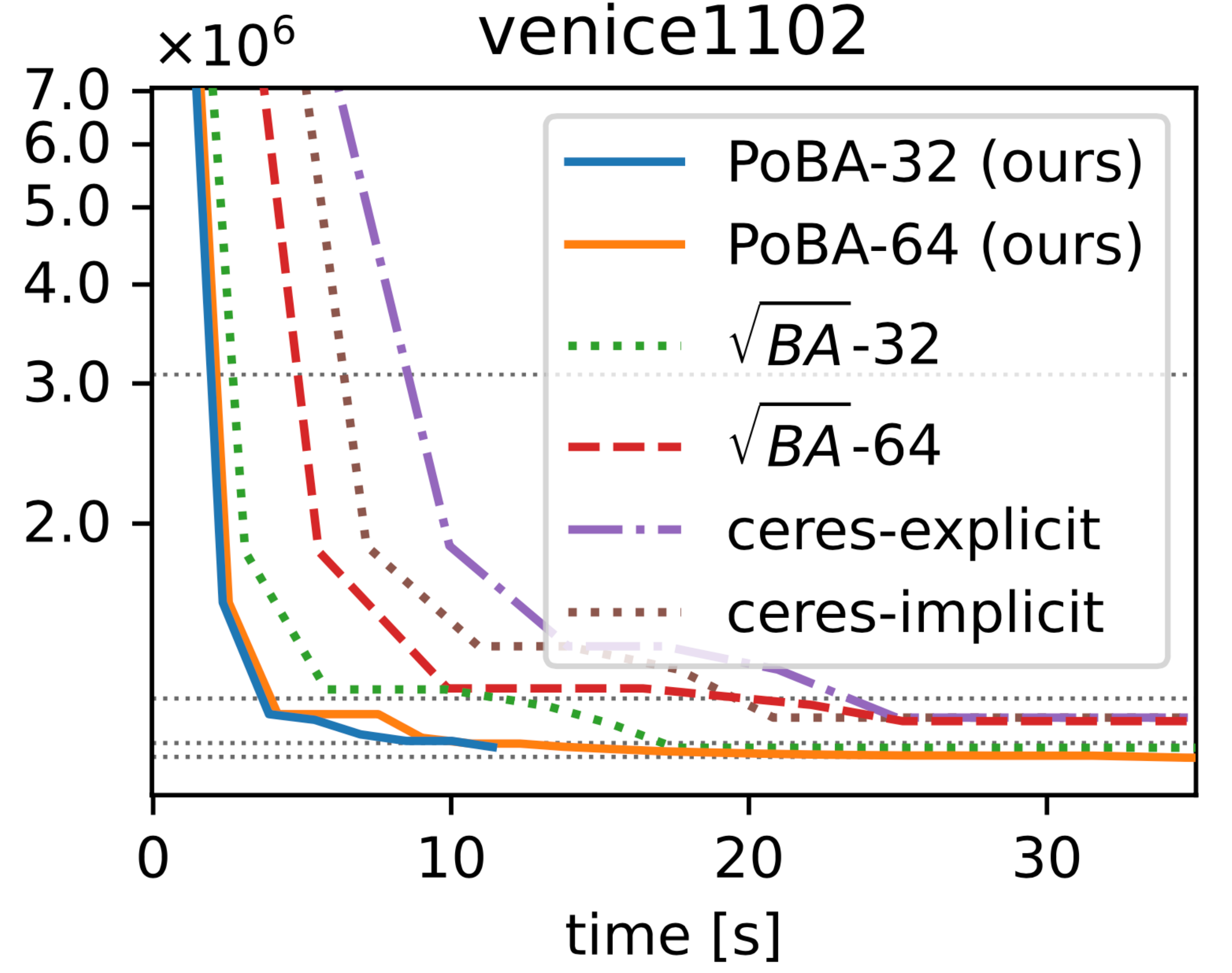}
\caption*{}
\label{fig:subim2}
\end{subfigure}
\caption{Convergence plots of \textit{Ladybug-1197} (left) from BAL dataset with $1197$ poses and \textit{Venice-1102} (right) from BAL dataset with $1102$ poses. Fig. \ref{fig:pangolin} shows a visualization of 3D landmarks and camera poses for these problems. The dotted lines correspond to cost thresholds for the tolerances $\tau \in \{0.1, 0.01, 0.003, 0.001\}$.}
\label{fig:convergence_plots}
\end{figure*}

\begin{figure*}
\begin{center}
\includegraphics[scale=0.69]{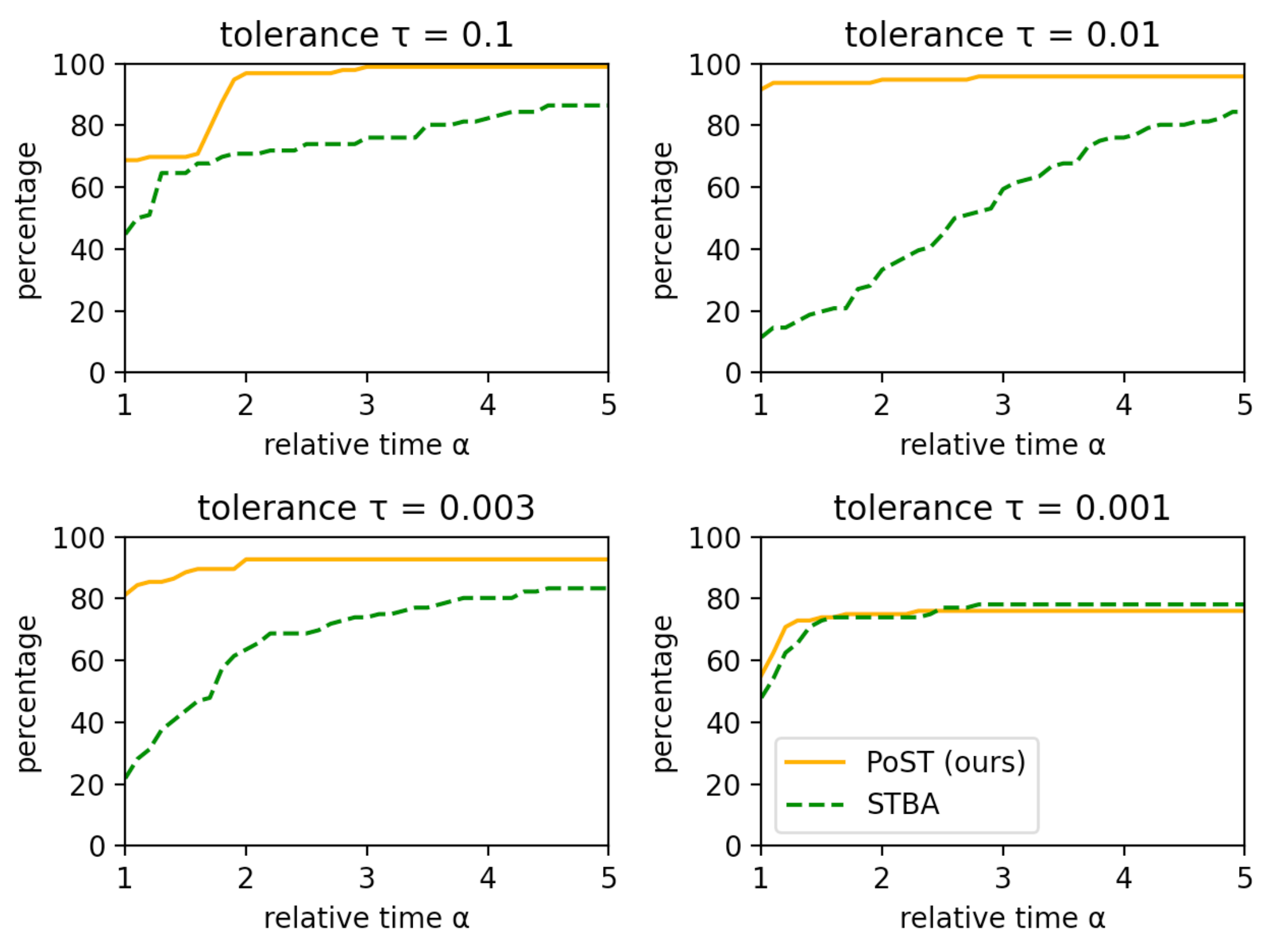}
\caption{Performance profiles for all BAL problems with stochastic framework. Our proposed solver PoST outperforms the challenging STBA across all accuracy tolerances $\tau \in \{0.1, 0.01, 0.003\}$, both in terms of speed and precision, and rivals STBA for $\tau = 0.001$.} \label{fig:stba_performance}
\end{center}
\end{figure*}

\subsection{Experimental Settings}\label{parameters}

\subsubsection*{Dataset.}
For our extensive evaluation we use all $97$ bundle adjustment problems from the BAL project page. They are divided within five problems families. \textit{Ladybug} is composed with images captured by a vehicle with regular rate. Images of \textit{Venice}, \textit{Trafalgar} and \textit{Dubrovnik} come from Flickr.com and have been saved as skeletal sets \cite{key-1}. Recombination of these problems with additional leaf images leads to the \textit{Final} family. Details about these problems can be found in Appendix.

\subsubsection*{LM loop.}\label{baseline_comparison} \textit{PoBA} is in line with the implementation \cite{demmel2021rootba} and with Ceres. Starting with damping parameter $10^{-4}$ we update $\lambda$ depending on the success or failure of the LM loop. We set the maximal number of LM iterations to $50$, terminating earlier if a relative function tolerance of $10^{-6}$ is reached. Concerning (\ref{x_approx}) and (\ref{qcriterion}) we set the maximal number of inner iterations to $20$ and a threshold $\epsilon = 0.01$. Ceres and $\sqrt{BA}$ use same forcing sequence for the inner CG loop, where the maximal number of iterations is set to $500$. We add a small Gaussian noise to disturb initial landmark and camera positions.

\begin{figure*}
\begin{subfigure}{0.50\textwidth}
\includegraphics[scale=0.36]{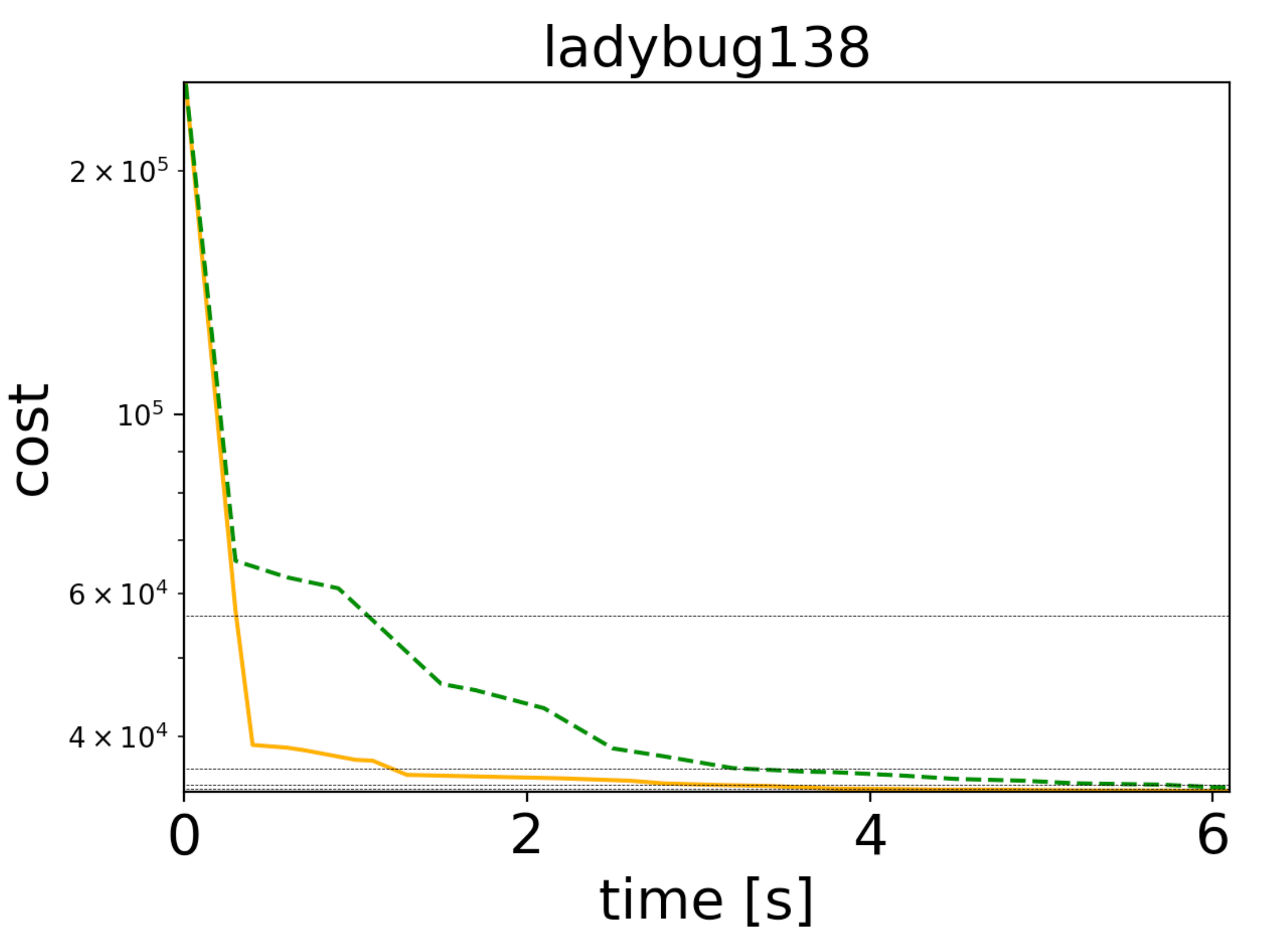} 
\caption*{}
\label{fig:subim1}
\end{subfigure}
\begin{subfigure}{0.50\textwidth}
\includegraphics[scale=0.36]{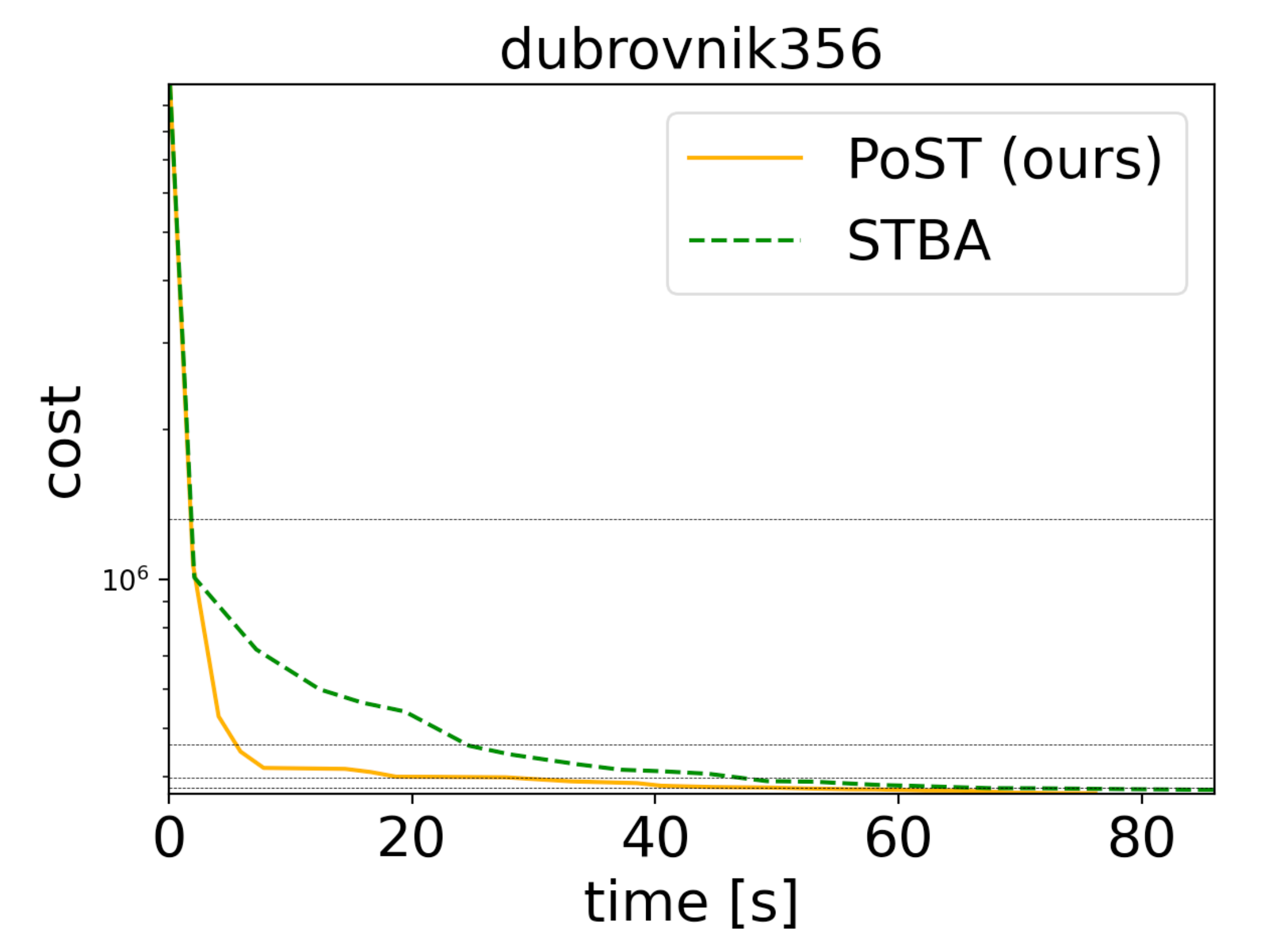}
\caption*{}
\label{fig:subim2}
\end{subfigure}
\vspace*{-12mm}
\caption{Convergence plots of \textit{Ladybug-138} (left) from BAL dataset with $138$ poses and \textit{Dubrovnik-356} (right) from BAL dataset with $356$ poses. The dotted lines correspond to cost thresholds for the tolerances $\tau \in \{0.1, 0.01, 0.003, 0.001\}$.}
\label{fig:convergence_plots_stba}
\end{figure*}

\subsection{Analysis}
\Cref{fig:performance} shows the performance profiles for all BAL datasets with tolerances $\tau \in \{0.1,0.01,0.003,0.001\}$. For $\tau = 0.1$ and $\tau = 0.01$ \textit{PoBA-$64$} clearly outperforms all challengers both in terms of runtime and accuracy. \textit{PoBA-$64$} remains clearly the best solver for the excellent accuracy $\tau = 0.003$ until a high relative time $\alpha = 4$. For higher relative time it is competitive with $\sqrt{BA}-32$ and still outperforms all other challengers. Same conclusion can be drawn from the convergence plot of two differently sized BAL problems (see \Cref{fig:convergence_plots}).
\Cref{fig:consumption} highlights the low memory consumption of \textit{PoBA} with respect to its challengers for all BAL problems. Whatever the size of the problem \textit{PoBA} is much less memory-consuming than $\sqrt{BA}$ and Ceres. Notably it requires almost five times less memory than $\sqrt{BA}$ and almost twice less memory than Ceres-implicit and Ceres-explicit.

\subsection{Power Stochastic Bundle Adjustment (PoST)}\label{stba}
\subsubsection*{Stochastic Bundle Adjustment.} 
STBA decomposes the reduced camera system into clusters inside the Levenberg-Marquardt iterations. The per-cluster linear sub-problems are then solved in parallel with dense $LL^{\top}$ factorization due to the dense connectivity inside camera clusters. 
As shown in \cite{key-4} this approach outperforms the baselines in terms of runtime and scales to very large BA problems, where it can even be used for distributed optimization. In the following we show that replacing the sub-problem solver with our Power Bundle Adjustment can significantly boost runtime even further.

We extend STBA\footnote{\url{https://github.com/zlthinker/STBA}} by incorporating our solver instead of the dense $LL^{\top}$ factorization. Each subproblem is then solved with a power series expansion of the inverse Schur complement with the same parameters as in Section~\ref{baseline_comparison}. In accordance to \cite{key-4} we set the maximal cluster size to $100$ and the implementation is written in double in C++.

\vspace*{-3mm}
\subsubsection*{Analysis.}
\vspace*{-2mm}
\Cref{fig:stba_performance} presents the performance profiles with all BAL problems for different tolerances $\tau$. Both solvers have similar accuracy for $\tau = 0.001$. For $\tau \in \{0.1, 0.01, 0.003\}$, PoST clearly outperforms STBA both in terms of runtime and accuracy, most notably for $\tau = 0.01$. Same observations are done when we plot the convergence for differently sized BAL problems (see \Cref{fig:convergence_plots_stba}).

\vspace*{-2mm}
\section{Conclusion}
\vspace*{-2mm}
We introduce a new class of large-scale bundle adjustment solvers that makes use of a power expansion of the inverse Schur complement. We prove the theoretical validity of the proposed approximation and the convergence of this solver. Moreover, we experimentally confirm that the proposed power series representation of the inverse Schur complement outperforms competitive iterative solvers in terms of speed, accuracy, and memory-consumption. Last but not least, we show that the power series representation can complement distributed bundle adjustment methods to significantly boost its performance for large-scale 3D reconstruction. 

\vspace*{-2mm}
\subsection*{Acknowledgement}
This work was supported by the ERC Advanced Grant SIMULACRON, the Munich Center for Machine Learning, the EPSRC Programme Grant VisualAI EP/T028572/1, and the DFG projects WU 959/1-1 and CR 250 20-1 “Splitting Methods for 3D Reconstruction and SLAM".


\begin{thebibliography}{10}

\bibitem{key-1}S. Agarwal, N. Snavely, S. M. Seitz, and
R. Szeliski. Bundle adjustment in the large. In \textit{European Conference on Computer Vision (ECCV)}, pages 29-42. Springer, 2010.

\bibitem{byrod2010conjugate}M. Byr{\"o}d, K. {\AA}str{\"o}m. Conjugate gradient bundle adjustment. In \textit{European Conference on Computer Vision (ECCV)}, 2010.

\bibitem{key-14}E. A. Coddington, N. Levinson. Theory of Ordinary Differential Equations. McGraw–Hill, 1955.

\bibitem{demmel2021rootba}N. Demmel, C. Sommer, D. Cremers, V. Usenko. Square Root Bundle Adjustment for Large-Scale Reconstruction. In \textit{Computer Vision and Pattern Recognition (CVPR)}, 2021.

\bibitem{demmel2021rootvo}N. Demmel, D. Schubert, C. Sommer, D. Cremers, V. Usenko. Square Root Marginalization for Sliding-Window Bundle Adjustment. In \textit{International Conference on Computer Vision (ICCV)}, 2021.

\bibitem{key-10}E. D. Dolan, and J. J. More. Benchmarking optimization software with performance profiles. In \textit{Mathematical programming} 91(2), pages 201–213, 2002.

\bibitem{eigenweb}G. Guennebaud, and B. Jacob, et al. \textit{Eigen v3}, \url{http://eigen.tuxfamily.org}, 2010.

\bibitem{key-8}M. R. Hestenes, and E. Stiefel. Methods of conjugate gradients for solving linear systems. In \textit{Journal of research of the National Bureau of Standards} 49(6), pages 409-436, 1952.

\bibitem{key-9}A. Kushal, and S. Agarwal. Visibility based preconditioning for bundle adjustment. In \textit{Conference on Computer Vision and Pattern Recognition (CVPR)}, 2012.

\bibitem{key-17}M. Lourakis, A. A. Argyros. Is levenberg-marquardt the most efficient optimization algorithm for implementing bundle adjustment? In \textit{International Conference on Computer Vision (ICCV)}, 2005. 

\bibitem{nash}S. G. Nash, A Survey of Truncated Newton Methods, Journal of Computational and Applied Mathematics, 124(1-2), 45-59, 2000.

\bibitem{nash2}S. G. Nash, A. Sofer, Assessing A Search Direction Within A Truncated Newton Method, Operation Research Letters 9(1990) 219-221.

\bibitem{ren2021megba}J. Ren, W. Liang, R. Yan, L. Mai, X. Liu . MegBA: A High-Performance and Distributed Library for Large-Scale Bundle Adjustment. In \textit{European Conference on Computer Vision (ECCV)}, 2022. 

\bibitem{key-6}Y. Saad. Itervative methods for sparse linear systems, 2nd ed. In \textit{SIAM}, Philadelpha, PA, 2003.

\bibitem{key-13}L. Trefethen, D. Bau. Numerical linear algebra. SIAM, 1997.

\bibitem{key-2}B. Triggs, P. F. McLauchlan, R. I. Hartley, and A. W. Fitzgibbon. Bundle adjustment-a modern synthesis. In \textit{International workshop on vision algorithms}, pages 298-372. Springer, 1999.

\bibitem{key-3}S. Weber, N. Demmel, and D. Cremers. Multidirectional conjugate gradients for scalable bundle adjustment. In \textit{German Conference on Pattern Recognition (GCPR)}, pages 712-724. Springer, 2021.

\bibitem{key-7}S. J. Wright, and J. N. Holt. An inexact Levenberg-Marquardt method for large sparse
nonlinear least squares. In \textit{J. Austral. Math. Soc. Ser. B 26}, pages 387-403, 1985.

\bibitem{key-15}C. Zach. Robust bundle adjustment revisited. In \textit{European Conference on Computer Vision (ECCV)}, 2014.

\bibitem{key-12}F. Zhang. The Schur complement and its applications. In Numerical Methods and Algorithms. Vol. 4, Springer, 2005.

\bibitem{key-5}Q. Zheng, Y. Xi, and Y. Saad. A power Schur complement low-rank correction preconditioner for general sparse linear systems. In {SIAM Journal on Matrix Analysis and Applications}, 2021.

\bibitem{key-4}L. Zhou, Z. Luo, M. Zhen, T. Shen, S. Li, Z. Huang, T. Fang, and L. Quan. Stochastic bundle adjustment for efficient and scalable 3d reconstruction. In {European Conference on Computer Vision (ECCV)}, 2020.



\bibitem{key-18}https://spectralib.org




\end{thebibliography}



\newpage
\onecolumn
\appendix
\vspace*{\stretch{1.0}}
\begin{center}
    \Large\textbf{Power Bundle Adjustment for Large-Scale 3D Reconstruction \\ Appendix}\\
\end{center}
\vspace*{\stretch{1.0}}


In this supplementary material we provide additional details to augment the content of the main paper. Section \ref{sec:supp-prop1-proof} contains a proof of Proposition 1 in the main paper. In Section \ref{sec:consistence-noises} we evaluate different levels of noises to highlight the consistence of our solver. In Section \ref{sec:detailed-pp} we tabulate the percentage of solved problems of the performance profiles (Sec. 5.2.) for each tolerance $\tau \in \{0.1, 0.01, 0.03, 0.001\}$ and for each solver. In Section \ref{sec:supp-problems} we list the evaluated problems from the BAL dataset.

\section{Proof of Proposition 1}
\label{sec:supp-prop1-proof}

Firstly, simple product expansion gives
\begin{equation}\label{eq}
    (I - M)(I + ... + M^{i}) = I - M^{i+1} \, .
\end{equation}
Since the spectral norm is sub-multiplicative and
\begin{equation}
    \lVert M \rVert < 1 \, ,
\end{equation}
it is straightforward that
\begin{equation}
    \lVert M^{i} \rVert \leq \lVert M \rVert^{i} \underset{ i \to \infty } \longrightarrow 0 \, .
\end{equation}
Thus,
\begin{equation}
    M^{i} \underset{ i \to \infty } \longrightarrow \boldsymbol{0} \, .
\end{equation}
Taking the limit of both sides in (\ref{eq}) gives (1).

Secondly, 
\begin{equation}
    R = \sum_{i=m+1}^{\infty} M^{i} = M^{m+1}\sum_{i=0}^{\infty}M^{i} = M^{m+1}(I-M)^{-1} \, .
\end{equation}
It follows that
\begin{equation}
    \lVert R \rVert = \lVert M^{m+1}\sum_{i=0}^{\infty} M^{i} \rVert \leq \lVert M \rVert ^{m+1}\sum_{i=0}^{\infty}\lVert M \rVert ^{i} \, .
\end{equation}
Since $\lVert M \rVert < 1$ we have
\begin{equation}
    \sum_{i=0}^{\infty}\lVert M \rVert ^{i} = \frac{1}{1 - \lVert M \rVert }\, ,
\end{equation}
which directly leads to the inequality
\begin{equation}
    \lVert R \rVert \leq \frac{\lVert M \rVert ^{m+1}}{1-\lVert M \rVert} \, .
\end{equation}

\section{Consistence}
\label{sec:consistence-noises}

In Sec. 5.2. initial landmark and camera positions are perturbated with a small Gaussian noise $(m,\sigma) = (0,0.01)$. We observe that the relative performance of solvers is similar for different noise levels. Fig. \ref{fig:performance_1} and \ref{fig:performance_2} illustrate the consistence of our results with different initial noises $\sigma = 0.05 $ and $\sigma = 0.1$. 

\begin{figure*}
\begin{center}
\includegraphics[scale=0.45]{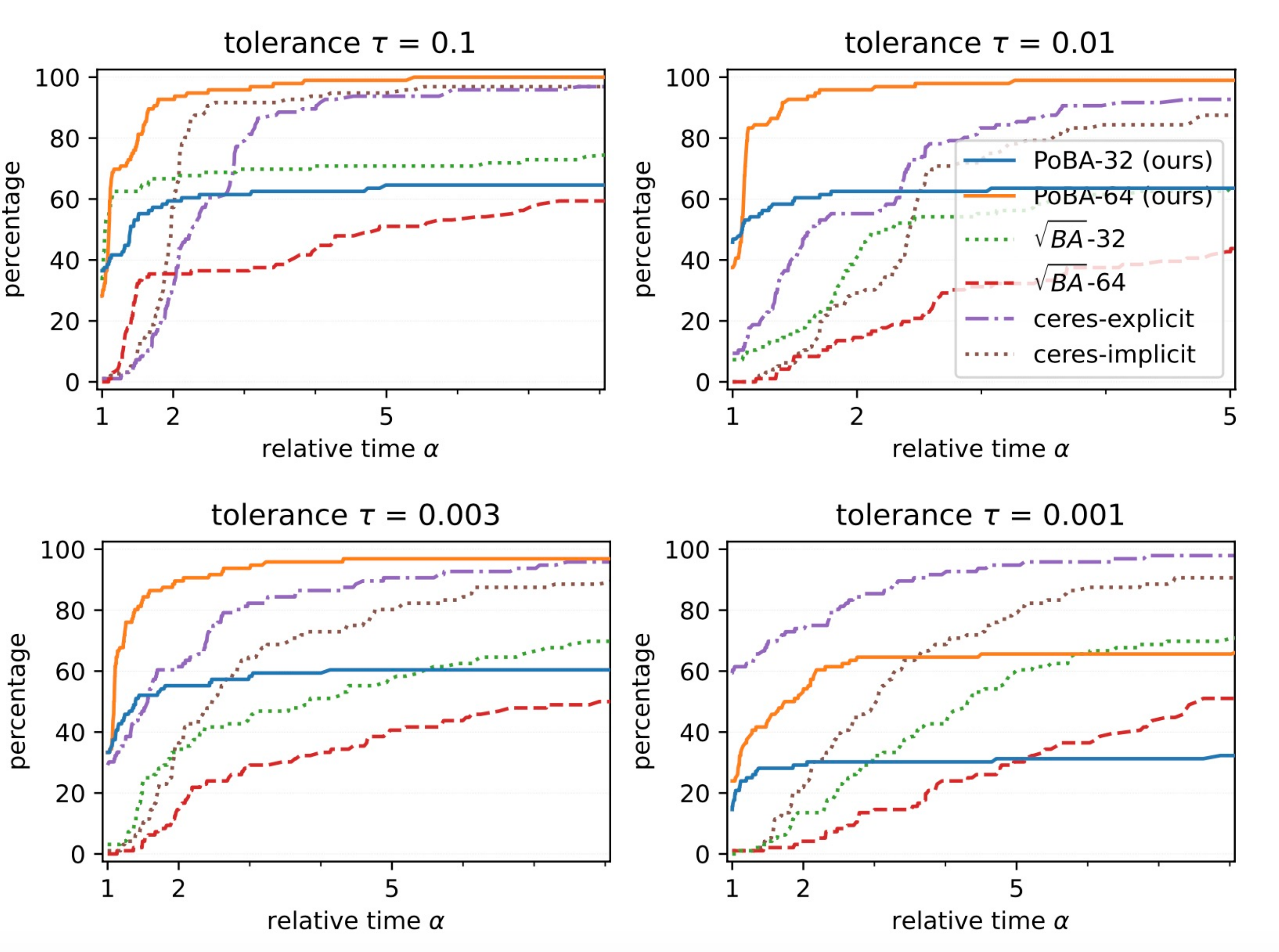}
\caption{Performance profiles for all BAL problems show the percentage of problems solved to a given accuracy tolerance $\tau \in \{0.1, 0.01, 0.003, 0.001\}$ with relative runtime $\alpha$. Initial landmark and camera positions are disturbed with a Gaussian noise $(0,0.05)$.}
\label{fig:performance_1}
\end{center}
\end{figure*}

\begin{figure*}
\begin{center}
\includegraphics[scale=0.45]{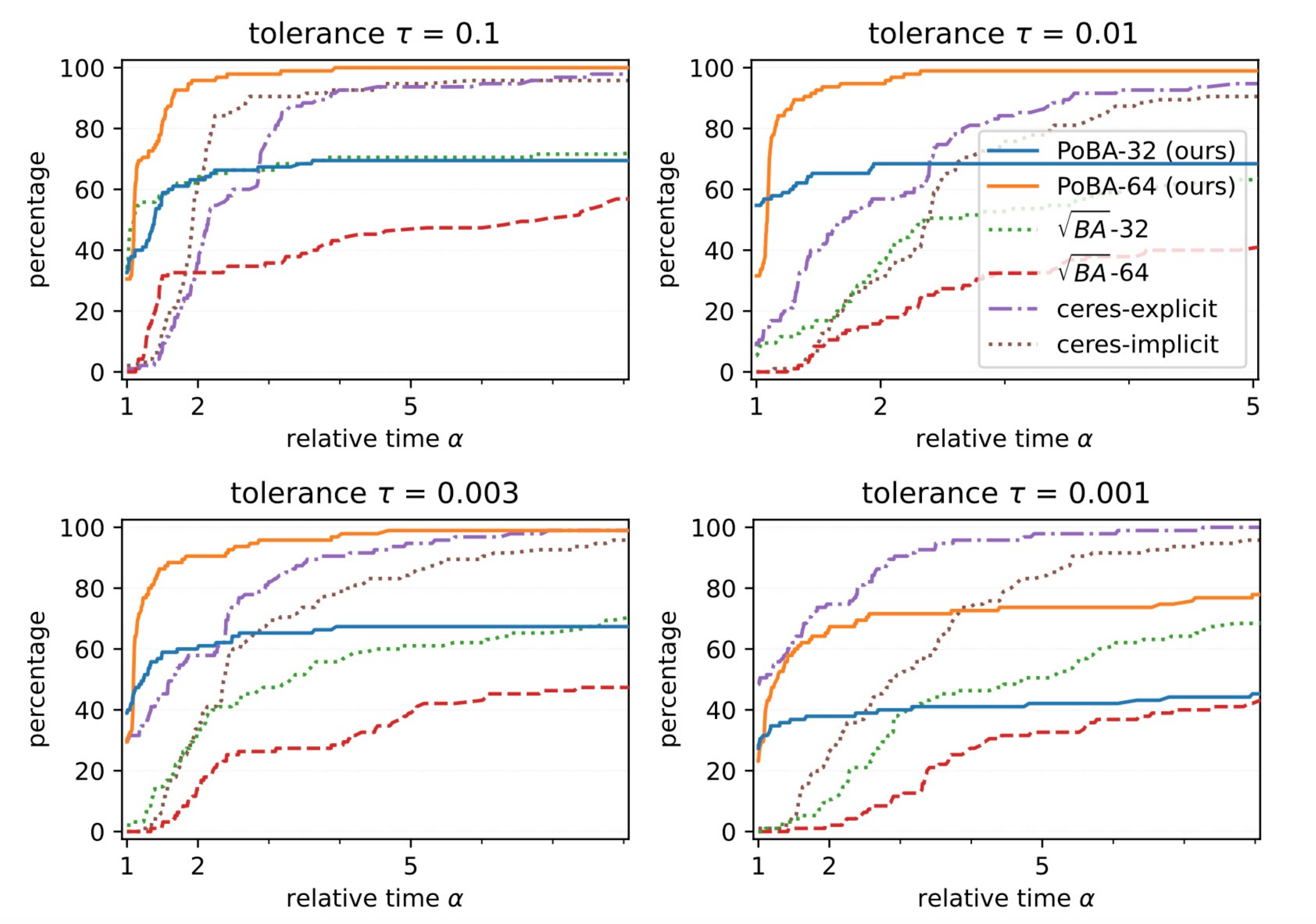}
\caption{Performance profiles for all BAL problems show the percentage of problems solved to a given accuracy tolerance $\tau \in \{0.1, 0.01, 0.003, 0.001\}$ with relative runtime $\alpha$. Initial landmark and camera positions are disturbed with a Gaussian noise $(0,0.1)$.}
\label{fig:performance_2}
\end{center}
\end{figure*}

\newpage

\clearpage
\onecolumn

\section{Tables of solved problems associated to the performance profiles}
\label{sec:detailed-pp}

\begin{table}[h]
 \begin{center}
 \begin{tabular}{ |p{2.4cm} || p{1cm}|p{1cm}|p{1.1cm}|  }
    \hline
    Solver & $\alpha = 1$ & $\alpha = 3$ & $\alpha = \infty$ \\
    \hline
    PoBA-$64$ (ours) & $7\%$ & \underline{$100\%$} & \underline{$100\%$} \\
    PoBA-$32$ (ours) & \underline{$62\%$} & $88\%$ & $88\%$ \\
    $\sqrt{BA}$-$64$ & $0\%$ & $60\%$ & $81\%$ \\
    $\sqrt{BA}$-$32$ & $31\%$ & $90\%$ & $98\%$ \\
    ceres-explicit & $0\%$ & $49\%$ & $95\%$ \\
    ceres-implicit & $0\%$ & $81\%$ & $95\%$ \\
    \hline
 \end{tabular}
\hfil
 \begin{tabular}{ |p{2.4cm} || p{1cm}|p{1cm}|p{1.1cm}|  }
    \hline
    Solver & $\alpha = 1$ & $\alpha = 3$ & $\alpha = \infty$ \\
    \hline
    PoBA-$64$ (ours) & $18\%$ & \underline{$98\%$} & \underline{$98\%$} \\
    PoBA-$32$ (ours) & \underline{$60\%$} & $84\%$ & $84\%$ \\
    $\sqrt{BA}$-$64$ & $0\%$ & $62\%$ & $79\%$ \\
    $\sqrt{BA}$-$32$ & $22\%$ & $83\%$ & $97\%$ \\
    ceres-explicit & $2\%$ & $66\%$ & $90\%$ \\
    ceres-implicit & $0\%$ & $80\%$ & $90\%$ \\
    \hline
 \end{tabular}
 
\medskip

 \begin{tabular}{ |p{2.4cm} || p{1cm}|p{1cm}|p{1.1cm}|  }
    \hline
    Solver & $\alpha = 1$ & $\alpha = 3$ & $\alpha = \infty$ \\
    \hline
    PoBA-$64$ (ours) & $20\%$ & \underline{$90\%$} & $93\%$ \\
    PoBA-$32$ (ours) & \underline{$44\%$} & $75\%$ & $79\%$ \\
    $\sqrt{BA}$-$64$ & $0\%$ & $58\%$ & $84\%$ \\
    $\sqrt{BA}$-$32$ & $22\%$ & \underline{$90\%$} & \underline{$98\%$} \\
    ceres-explicit & $14\%$ & $71\%$ & $90\%$ \\
    ceres-implicit & $0\%$ & $66\%$ & $91\%$ \\
    \hline
 \end{tabular}
\hfil
 \begin{tabular}{ |p{2.4cm} || p{1cm}|p{1cm}|p{1.1cm}|  }
    \hline
    Solver & $\alpha = 1$ & $\alpha = 3$ & $\alpha = \infty$ \\
    \hline
    PoBA-$64$ (ours) & $13\%$ & $52\%$ & $58\%$ \\
    PoBA-$32$ (ours) & $8\%$ & $26\%$ & $27\%$ \\
    $\sqrt{BA}$-$64$ & $0\%$ & $63\%$ & $85\%$ \\
    $\sqrt{BA}$-$32$ & $21\%$ & \underline{$88\%$} & \underline{$98\%$} \\
    ceres-explicit & \underline{$54\%$} & $83\%$ & $90\%$ \\
    ceres-implicit & $4\%$ & $75\%$ & $90\%$ \\
    \hline
 \end{tabular}
 \caption{Percentage of solved problems of the performance profiles (Sec. 5.2.) for each solver and for tolerance $\tau = 0.1$ (upper left), $\tau = 0.01$ (upper right), $\tau = 0.003$ (lower left) and $\tau = 0.001$ (lower right). We conclude that PoBA is particularly well suited for very fast/low-accurate ($\tau = 0.1$), fast/medium-accurate ($\tau = 0.01$) and slow/high-accurate ($\tau = 0.003$) applications.}
 \end{center}
 \end{table}




\section{Problems Table}
\label{sec:supp-problems}{

\makeatletter
\newcommand*\ExpandableInput[1]{\input{#1} }
\makeatother
\setlength{\tabcolsep}{2.3em}
\setlength{\LTcapwidth}{0.99\textwidth}
\begin{longtable}{l r r r}%
\label{tab:problem-size}
\endfirsthead
\endhead
\toprule
&\multicolumn{1}{c}{cameras}&\multicolumn{1}{c}{landmarks}&\multicolumn{1}{c}{observations} \\
\midrule
ladybug-49&49&7,766&31,812\\%
ladybug-73&73&11,022&46,091\\%
ladybug-138&138&19,867&85,184\\%
ladybug-318&318&41,616&179,883\\%
ladybug-372&372&47,410&204,434\\%
ladybug-412&412&52,202&224,205\\%
ladybug-460&460&56,799&241,842\\%
ladybug-539&539&65,208&277,238\\%
ladybug-598&598&69,193&304,108\\%
ladybug-646&646&73,541&327,199\\%
ladybug-707&707&78,410&349,753\\%
ladybug-783&783&84,384&376,835\\%
ladybug-810&810&88,754&393,557\\%
ladybug-856&856&93,284&415,551\\%
ladybug-885&885&97,410&434,681\\%
ladybug-931&931&102,633&457,231\\%
ladybug-969&969&105,759&474,396\\%
ladybug-1064&1,064&113,589&509,982\\%
ladybug-1118&1,118&118,316&528,693\\%
ladybug-1152&1,152&122,200&545,584\\%
ladybug-1197&1,197&126,257&563,496\\%
ladybug-1235&1,235&129,562&576,045\\%
ladybug-1266&1,266&132,521&587,701\\%
ladybug-1340&1,340&137,003&612,344\\%
ladybug-1469&1,469&145,116&641,383\\%
ladybug-1514&1,514&147,235&651,217\\%
ladybug-1587&1,587&150,760&663,019\\%
ladybug-1642&1,642&153,735&670,999\\%
ladybug-1695&1,695&155,621&676,317\\%
ladybug-1723&1,723&156,410&678,421\\%
\midrule
&\multicolumn{1}{c}{cameras}&\multicolumn{1}{c}{landmarks}&\multicolumn{1}{c}{observations} \\
\midrule
trafalgar-21&21&11,315&36,455\\%
trafalgar-39&39&18,060&63,551\\%
trafalgar-50&50&20,431&73,967\\%
trafalgar-126&126&40,037&148,117\\%
trafalgar-138&138&44,033&165,688\\%
trafalgar-161&161&48,126&181,861\\%
trafalgar-170&170&49,267&185,604\\%
trafalgar-174&174&50,489&188,598\\%
trafalgar-193&193&53,101&196,315\\%
trafalgar-201&201&54,427&199,727\\%
trafalgar-206&206&54,562&200,504\\%
trafalgar-215&215&55,910&203,991\\%
trafalgar-225&225&57,665&208,411\\%
trafalgar-257&257&65,131&225,698\\%
\midrule
&\multicolumn{1}{c}{cameras}&\multicolumn{1}{c}{landmarks}&\multicolumn{1}{c}{observations} \\
\midrule
dubrovnik-16&16&22,106&83,718\\%
dubrovnik-88&88&64,298&383,937\\%
dubrovnik-135&135&90,642&552,949\\%
dubrovnik-142&142&93,602&565,223\\%
dubrovnik-150&150&95,821&567,738\\%
dubrovnik-161&161&103,832&591,343\\%
dubrovnik-173&173&111,908&633,894\\%
dubrovnik-182&182&116,770&668,030\\%
dubrovnik-202&202&132,796&750,977\\%
dubrovnik-237&237&154,414&857,656\\%
dubrovnik-253&253&163,691&898,485\\%
dubrovnik-262&262&169,354&919,020\\%
dubrovnik-273&273&176,305&942,302\\%
dubrovnik-287&287&182,023&970,624\\%
dubrovnik-308&308&195,089&1,044,529\\%
dubrovnik-356&356&226,729&1,254,598\\%
\midrule
&\multicolumn{1}{c}{cameras}&\multicolumn{1}{c}{landmarks}&\multicolumn{1}{c}{observations} \\
\midrule
venice-52&52&64,053&347,173\\%
venice-89&89&110,973&562,976\\%
venice-245&245&197,919&1,087,436\\%
venice-427&427&309,567&1,695,237\\%
venice-744&744&542,742&3,054,949\\%
venice-951&951&707,453&3,744,975\\%
venice-1102&1,102&779,640&4,048,424\\%
venice-1158&1,158&802,093&4,126,104\\%
venice-1184&1,184&815,761&4,174,654\\%
venice-1238&1,238&842,712&4,286,111\\%
venice-1288&1,288&865,630&4,378,614\\%
venice-1350&1,350&893,894&4,512,735\\%
venice-1408&1,408&911,407&4,630,139\\%
venice-1425&1,425&916,072&4,652,920\\%
venice-1473&1,473&929,522&4,701,478\\%
venice-1490&1,490&934,449&4,717,420\\%
venice-1521&1,521&938,727&4,734,634\\%
venice-1544&1,544&941,585&4,745,797\\%
venice-1638&1,638&975,980&4,952,422\\%
venice-1666&1,666&983,088&4,982,752\\%
venice-1672&1,672&986,140&4,995,719\\%
venice-1681&1,681&982,593&4,962,448\\%
venice-1682&1,682&982,446&4,960,627\\%
venice-1684&1,684&982,447&4,961,337\\%
venice-1695&1,695&983,867&4,966,552\\%
venice-1696&1,696&983,994&4,966,505\\%
venice-1706&1,706&984,707&4,970,241\\%
venice-1776&1,776&993,087&4,997,468\\%
venice-1778&1,778&993,101&4,997,555\\%
\midrule
&\multicolumn{1}{c}{cameras}&\multicolumn{1}{c}{landmarks}&\multicolumn{1}{c}{observations} \\
\midrule
final-93&93&61,203&287,451\\%
final-394&394&100,368&534,408\\%
final-871&871&527,480&2,785,016\\%
final-961&961&187,103&1,692,975\\%
final-1936&1,936&649,672&5,213,731\\%
final-3068&3,068&310,846&1,653,045\\%
final-4585&4,585&1,324,548&9,124,880\\%
final-13682&13,682&4,455,575&28,973,703\\%

\bottomrule
\caption{List of all 97 BAL problems including number of cameras, landmarks and observations. These are the problems evaluated for performance profiles in the main paper.}
\end{longtable}
}

\end{document}